\documentclass{article} 
\usepackage{iclr2026_conference,times}

\usepackage{amsmath,amsfonts,bm}

\def\eqref#1{equation~\ref{#1}}

\def\1{\bm{1}}

\DeclareMathAlphabet{\mathsfit}{\encodingdefault}{\sfdefault}{m}{sl}
\SetMathAlphabet{\mathsfit}{bold}{\encodingdefault}{\sfdefault}{bx}{n}

\usepackage{hyperref}
\usepackage{url}

\usepackage{amsmath}
\usepackage{amsthm}
\usepackage{amsfonts}
\usepackage{cleveref}
\usepackage{multirow}
\usepackage{subcaption}
\usepackage{wrapfig}
\usepackage{bbm}

\usepackage{tikz}
\usetikzlibrary{cd}
\usepackage{algpseudocode}
\usepackage{algorithm2e}
\usepackage{listings}
\usepackage{minted}
\usepackage{booktabs}
\usepackage{array}

\crefname{algocf}{alg.}{algs.}
\Crefname{algocf}{Algorithm}{Algorithms}

\title{Unbiased Gradient Estimation for Event Binning via Functional Backpropagation}

\author{Jinze Chen, Wei Zhai\thanks{Correspondence to: Wei Zhai \textless\url{wzhai056@ustc.edu.cn}\textgreater}~, Han Han, Tiankai Ma, Yang Cao, Bin Li, Zheng-Jun Zha\\
MoE Key Laboratory of Brain-inspired Intelligent Perception and Cognition,\\
University of Science and Technology of China\\
\texttt{\{chjz@mail.,wzhai056@,hanh@mail.,tiankaima@mail.\}ustc.edu.cn}\\
\texttt{\{forrest,binli,zhazj\}@ustc.edu.cn}}

\newcommand{\red}[1]{{#1}}
\newcommand{\eg}{\textit{e.g., }}

\newtheorem{theorem}{Theorem}

\newtheorem{corollary}{Corollary}
\newtheorem{definition}{Definition}

\iclrfinalcopy 
\begin{document}

\maketitle

\begin{abstract}
Event-based vision encodes dynamic scenes as asynchronous spatio-temporal spikes called events. To leverage conventional image processing pipelines, events are typically binned into frames. However, binning functions are discontinuous, which truncates gradients at the frame level and forces most event-based algorithms to rely solely on frame-based features. Attempts to directly learn from raw events avoid this restriction but instead suffer from biased gradient estimation due to the discontinuities of the binning operation, ultimately limiting their learning efficiency. To address this challenge, we propose a novel framework for unbiased gradient estimation of arbitrary binning functions by synthesizing weak derivatives during backpropagation while keeping the forward output unchanged. The key idea is to exploit integration by parts: lifting the target functions to functionals yields an integral form of the derivative of the binning function during backpropagation, where the cotangent function naturally arises. By reconstructing this cotangent function from the sampled cotangent vector, we compute weak derivatives that provably match long-range finite differences of both smooth and non-smooth targets. Experimentally, our method improves simple optimization-based egomotion estimation with 3.2\% lower RMS error and 1.57$\times$ faster convergence. On complex downstream tasks, we achieve 9.4\% lower EPE in self-supervised optical flow, and 5.1\% lower RMS error in SLAM, demonstrating broad benefits for event-based visual perception. Source code can be found at \url{https://github.com/chjz1024/EventFBP}.

\end{abstract}

\section{Introduction\label{sec:introduction}}
Event-based visual perception has recently emerged as a novel paradigm to encode highly dynamic scenes with spatio-temporal event spikes. This paradigm shift brings new opportunities to high-speed and low-latency visual information processing, such as fine-grained optical flow estimation \citep{zhu2018ev,gehrig2021raft,shiba2023event,hamann2024motion,luo2024efficient,wan2024event,wan_2025_iccv,han_2025_iccv}, high-speed robot localization \citep{vidal2018ultimate,zhou2021event,hines2025compact}, and blurless video generation \citep{pan2019bringing,rebecq2019high,tulyakov2021time,wu2024event,xu2025motion,liao2025efdgs}.

To leverage conventional image processing pipelines, a common practice is to apply data binning techniques to convert irregular events into dense frames \citep{gallego2020event}, as shown in \Cref{fig:eventvisbpbias}. However, binning functions are inherently discontinuous, which truncates gradients at the frame level and forces most event-based algorithms to rely solely on frame-based features. Existing approaches often resort to smooth binning to restore differentiability to raw events, but such surrogates inevitably introduce bias in the gradients, leading to suboptimal learning efficiency.

\begin{figure}[h]
    \centering
    \includegraphics[width=.91\linewidth]{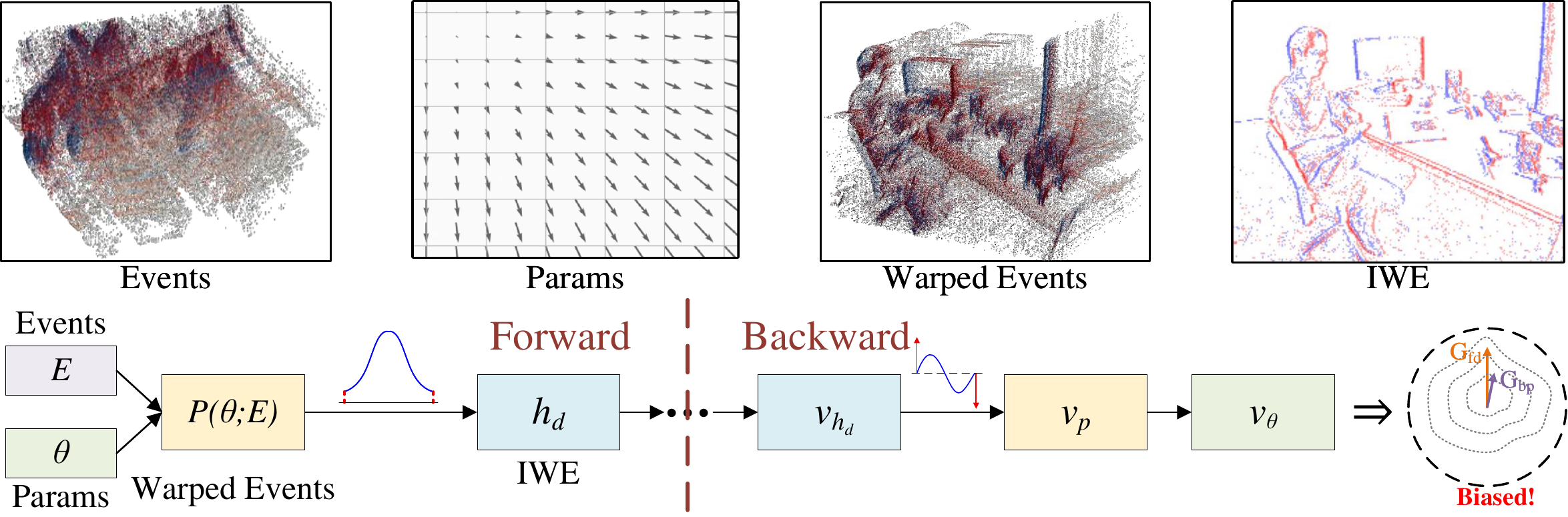}
    \caption{\red{\textbf{Event binning and the gradient bias problem.} \textbf{Top:} Spatio-temporal event clouds (Events) are geometrically warped using motion parameters (Params) and aggregated into an Image of Warped Events (IWE) using a binning function. \textbf{Bottom Left (Forward):} The warping function $P(\theta;E)$ transforms input events $E$ and parameters $\theta$ into warped coordinates. These are processed by a discontinuous binning function $h_d$ to produce the IWE. \textbf{Bottom Right (Backward):} The adjoint (cotangent vector) of the IWE, denoted as $v_{h_d}$, is propagated back to update parameters. However, the discontinuity of $h_d$ results in non-computable Dirac delta functions when computing the gradient $v_p$ for the warped events. \textbf{Result:} The computed backpropagation gradient $G_{bp}$ deviates from the true finite difference gradient $G_{fd}$, shown in the contour plot as a "Biased!" estimation.}} 
    \label{fig:eventvisbpbias}
\end{figure}

This challenge reflects a broader issue in learning with discontinuous nonlinearities in neuromorphic computing, where spiking neuron models introduce non-differentiable functions \citep{eshraghian2023training}. Existing solutions, such as surrogate gradients \citep{neftci2019surrogate} or straight-through estimators \citep{yin2018understanding}, provide heuristic gradients, but they lack unbiasedness guarantees.

From a mathematical perspective, weak derivatives \citep{kuttler2017weak} generalize classical derivatives to discontinuous functions. While their pointwise values may be non-computable (\eg Dirac delta), their integrals are well-defined and can be computed via integration by parts. This observation is crucial: if we can computationally match the integral of weak derivatives, then the resulting gradient estimation recovers unbiasedness.

\begin{wrapfigure}{l}{.45\linewidth}
    \begin{center}
        \includegraphics[width=\linewidth]{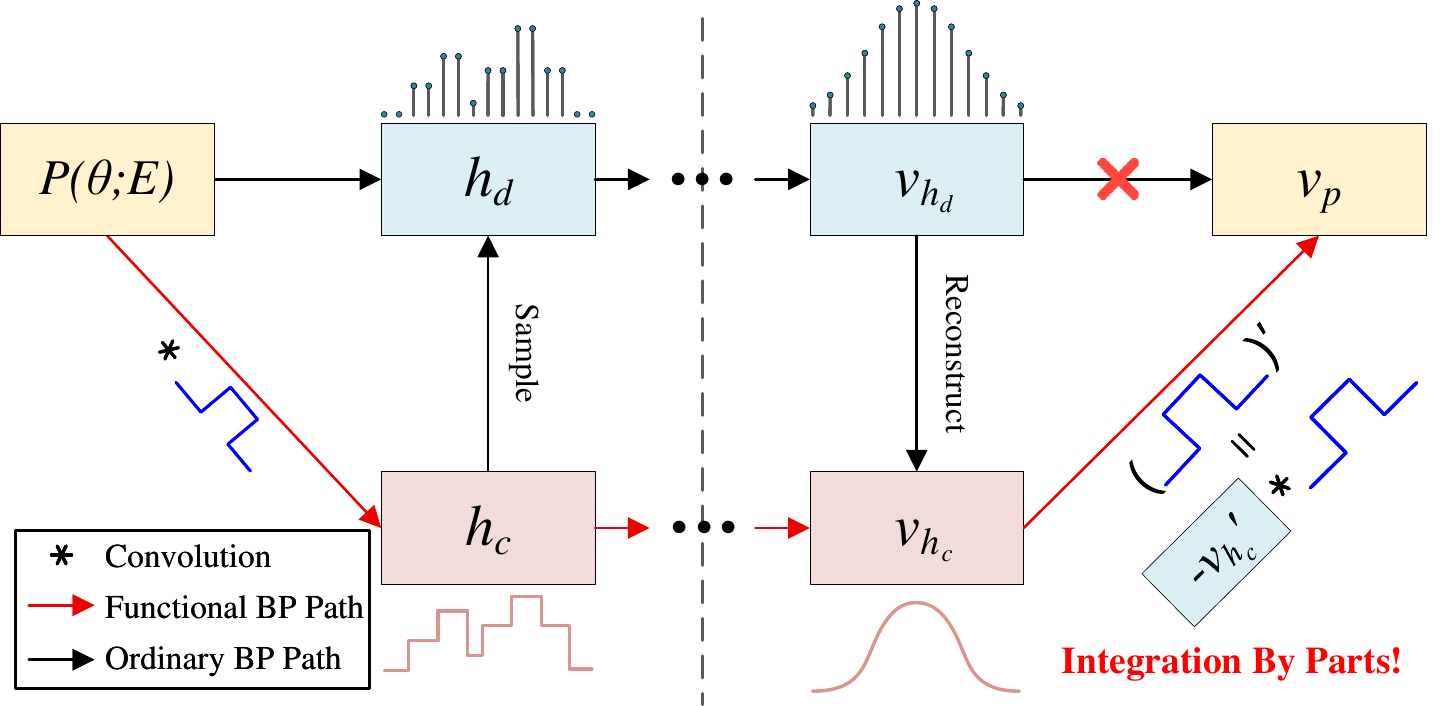}
    \end{center}
    \caption{\red{\textbf{The proposed Functional Backpropagation (FBP) framework.} To resolve discontinuities in the ordinary path (top, discrete binning $h_d$), we lift the operation to a functional space (bottom, continuous binning $h_c$). FBP bridges the two by reconstructing the continuous cotangent function $v_{h_c}$ from the discrete samples $v_{h_d}$. Using integration by parts, we replace the undefined Dirac delta evaluation with a convolution ($*$) of $v_{h_c}$ and the kernel derivative, synthesizing an unbiased gradient $v_p$.}}
    \label{fig:functionbp}
\end{wrapfigure}

To find such an integral, we dive into the backpropagation process and find that if we lift the binning function to the space of functionals, an integral of the respective gradient and the cotangent function naturally arises, which is provably sampled by the cotangent vector, as shown in \Cref{fig:functionbp}. \red{Because cotangent functions naturally encode continuous motion flow, smooth priors can be utilized to reconstruct them from the cotangent vector irrespective of the forward pass. This allows us to derive an exact formula for the synthesized weak derivative of the binning function}, which is provably shown to approximate long-range finite differences of arbitrary target functions. It also applies to continuous soft binning functions, where the synthesized gradients may help to skip local extremes.

To validate all the above claims, we first conduct analytical studies on simple tasks of event-based motion estimation, with an impressive result on the convergence speed and increased accuracy. We then conduct experiments on real-world tasks of optical-flow estimation and SLAM to show its wide applicability.

To summarize, our main contributions are:
\begin{enumerate}
    \item We identify and formalize the fundamental issue of biased gradient estimation in event-based pipelines, stemming from the discontinuity of binning functions.
    \item We propose a novel functional backpropagation framework that lifts binning functions into functional space, enabling gradient computation via weak derivatives. By applying integration by parts, our method avoids Dirac deltas and restores unbiasedness.
    \item We derive an exact formula for the synthesized weak derivative, proving its equivalence to long-range finite differences for arbitrary functions, and show that it naturally generalizes to both discontinuous and smooth kernels.
    \item We demonstrate consistent improvements across tasks: 3.7\% lower RMS velocity error with 1.57$\times$ faster convergence in controlled settings, 9.4\% lower EPE in optical flow, and 5.1\% lower RMS trajectory error in SLAM.
\end{enumerate}

\section{Related Work\label{sec:relatedwork}}

\textbf{Binning in Event-Based Vision.}
In event-based vision, binning extends beyond its canonical form to support learning diverse visual representations. Two main types can be distinguished depending on where the information resides. The first focuses on binning weights, such as constants or event polarities \citep{maqueda2018event,luo2025learning}, latest event timestamps \citep{lagorce2016hots,ghosh2025event}, or relative event intervals \citep{pan2019bringing,teng2022nest}, which are primarily used to model irradiance change, motion, and gradients, respectively. These methods typically ignore binning gradients during learning and thus suffer from information loss. The second type focuses on binning locations, where events undergo parametric transformations, and parameters are recovered by optimizing sharpness metrics of the resulting event frames \citep{gallego2018unifying,gu2021spatio,Shiba24pami}. While this approach captures complex spatio-temporal dynamics, it is fundamentally limited by the non-differentiability of the binning function. Our work is mainly concerned with this latter category.

\textbf{Learning with Discontinuous Nonlinearities.}
There are three common approaches to restoring differentiability for discontinuous functions. The most direct is function relaxation \citep{huh2018gradient}, but it alters the output and may compromise its physical meaning. Straight-through estimators (STE) \citep{yin2018understanding} heuristically replace the gradient with a surrogate, while surrogate gradients (SG) \citep{neftci2019surrogate} introduce continuous relaxations of the true gradients. However, none of these methods guarantee unbiasedness, and their effectiveness often depends on domain expertise in choosing appropriate relaxations. Our method instead reconstructs a continuous relaxation of the cotangent function, providing a complementary alternative to surrogate gradients with the key property of unbiasedness.

\textbf{Weak and Functional Derivatives.}
Our method draws on functional analysis, in particular weak derivatives and functional derivatives. Weak derivatives are defined through smooth test functions, where integration by parts provides an alternative characterization \citep{kuttler2017weak}. Functional derivatives extend the notion of gradients to functionals, playing a central role in the calculus of variations \citep{frigyik2008introduction}. While these concepts are typically of theoretical interest due to computational intractability, we show that they naturally arise in backpropagation when lifting binning functions to the space of functionals, thereby enabling their practical use in learning.

\section{Methodology\label{sec:methodology}}
This section is organized in \red{four} parts: First, we extend ordinary backpropagation to functionals where an integral form naturally arises alongside the cotangent function, by deriving the chain rule for arbitrary functionals besides integrals encountered in the calculus of variations. Then we prove that an ordinary backpropagation pass samples the cotangent function with cotangent vectors, so that signal reconstruction techniques can be applied to compute the weak derivative of binning functions. \red{Subsequently, we present a pseudocode implementation to bridge the gap between the theoretical derivation and practical application in forward and backward modes. Finally, we experimentally validate that the method approximates long-range finite differences for various loss functions, and the formal proof of unbiasedness is included in the appendix.}


We start with a notational convention to distinguish between functionals and pointwise evaluation:
\begin{itemize}
    \item Square brackets are used for functionals acting on elements of a function space; \eg $f[u]$.
    \item Parentheses are used for pointwise evaluation of ordinary functions; \eg, $g(x)$.
\end{itemize}

\subsection{Functional Backpropagation}\label{sec:fbp}
A \emph{functional} is defined by a rule, which associates a number with a function. Let $\mathcal{H}(X):=(X\to\mathbb{R})$ represent the space of real-valued functions defined on $X$, then a \emph{functional} $f$ is defined as a mapping $u\in \mathcal{H}(X)\mapsto f[u]$, where $f\left[u\right]$ can itself be another function on space $\mathcal{H}\left(Y\right)$. In this definition, functionals can be chained and used to derive the chain rule.

Since functionals are functions that consume an entire function to give another function, their differentials are also functions. From the knowledge of the Fréchet derivative and the Riesz representation theorem in functional analysis, we have the following result:
\begin{definition}
    Assuming $f$ is a differentiable functional on vectors spaces from $\mathcal{H}(X)$ to $\mathcal{H}(Y)$, then the derivative of $f$ at $u\in \mathcal{H}(X)$ is represented by the derivative kernel function $\frac{\delta f[u](y)}{\delta u(x)}$, which satisfies the limit:
    \begin{equation}
        \lim_{\|\delta u\|\to0}\frac{\|f[u+\delta u]-f[u]-\int_X\frac{\delta f[u]}{\delta u(x)}\delta u(x)dx\|_{\mathcal{H}(Y)}}{\|\delta u\|_{\mathcal{H}(X)}}=0.
    \end{equation}
\end{definition}

With the above definition, the chain rule for functionals is described in terms of integrals.

\begin{theorem}\label{theorem:functionalchainrep}
    Let $\mathcal{H}(X)\overset{f}{\to}\mathcal{H}(Y)\overset{g}{\to}\mathcal{H}(Z)$ where $f$ and $g$ are differentiable functionals, then the composite functional $g\circ f$ is also differentiable and has the representation:
    \begin{equation}\label{eq:functionalchainrep}
        \frac{\delta g[f[u]](z)}{\delta u(x)}=\int_Y \frac{\delta g[f[u]](z)}{\delta f[u](y)}\frac{\delta f[u](y)}{\delta u(x)}dy.
    \end{equation}
\end{theorem}

When $X,Y,Z$ are finite sets, \Cref{theorem:functionalchainrep} becomes the familiar chain rule $J_{g\circ f}=J_g\cdot J_f$ where $J_\cdot$ is the finite-dimensional Jacobian matrix, since integration becomes summation on finite sets. So any composable differentiable mappings have a chain rule, where every ordinary function induces a Jacobian-vector product and every functional induces an integral transform with the derivative kernel function.

To compute the Jacobians, classical backpropagation recursively computes $vJ_{g\circ f}$ using the associative property of matrix multiplication on Jacobians as $(vJ_g)\cdot J_f$, where $v\in \mathcal{H}(Z)$ is called a cotangent vector \citep{shi2024stochastic}. The extension for infinite-dimensional functionals is a function $v(\cdot)\in\mathcal{H}(Z)$, and we call it the cotangent function. In analogy, we call the process of recursively computing the cotangent function using the chain rule \red{\textbf{Functional Backpropagation (FBP)}}.


\subsection{Synthesizing Weak Derivatives of The Binning Function}
An event-based algorithm accepts a group of events as input, defined as
\begin{equation}
 \mathcal{E}=\{e_i=(t_i,x_i,y_i,p_i)\}_{i=1}^{N_e},
\end{equation}
where $e_i$ represents an event with timestamp $t_i$, pixel location $(x_i,y_i)$, and polarity $p_i\in\{-1,+1\}$ representing sign of brightness change, and $N_e$ is the number of events. A parametric transform maps $\mathcal{E}$ to D-dimensional weighted points 
\begin{equation}
    \mathcal{P}=\{e_i'=(x_{i_1}',\cdots,x_{i_D}',w_i')\}_{i=1}^{N_e},
\end{equation}
where $x_{i_d}'$ is the transformed location and $w_i'$ is the weight. The binning function $\bm{h}$ is a map from $\mathcal{P}$ to $\mathbb{R}^{W_1\times\cdots\times W_D}$ where $W_d$ is the number of bins. The value at index $(j_1,\cdots,j_D)$ is defined as
\begin{equation}
    \label{eq:ndhistdef}
    h_{j_1,...,j_D}=\sum_{i=1}^{N_e}w_i'\prod_{d=1}^Dk_d(\frac{x_{i_d}'-j_d\Delta_d}{\Delta_d}),
\end{equation}
where $k_d(\cdot)$ is the binning kernel with finite support and $\Delta_d$ is the binning width. This section provides an intuitive derivation of the formula for synthesizing the weak derivative $\frac{\partial h_{j_1,\cdots,j_D}}{\partial x_{i_d}'}$. To simplify the derivation, our discussion is restricted to the $D=1$ case, but the result can be easily extended to arbitrary dimensions. A \red{structured Theorem-Proof} is provided in the appendix.

We restate the definition of 1D event binning function $\bm{h}$ with input points $\mathcal{P}$ as:
\begin{equation}
    \bm{h}(\mathcal{P})=(h_1(\mathcal{P}),\cdots,h_W(\mathcal{P})),\quad
    h_j(\mathcal{P})=\sum_{i=1}^{N_e}w_i'k(\frac{x_{i}'-j\Delta}{\Delta}).
\end{equation}

To learn the parameters, a scalar loss is constructed from the obtained $\bm{h}(\mathcal{P})$ using $f_d:\bm{h}(\mathcal{P})\mapsto f_d[\bm{h}[\mathcal{P}]]\in\mathbb{R}$. 
A continuous binning function $h\left(\mathcal{P}\right)\in\mathcal{H}\left(\mathbb{R}\right)$ can be obtained by replacing $j\Delta$ with a continuous parameter $x$:
\begin{equation}
    h[\mathcal{P}](x)=\sum_{i=1}^{N_e}w_i'k(\frac{x_i'-x}{\Delta}),
\end{equation}
A natural mapping exists from $h\left[\mathcal{P}\right]$ to $\bm{h}\left[\mathcal{P}\right]$ using the sampling functional $\mathcal{S}:h\left[\mathcal{P}\right]\mapsto \mathcal{S}[h\left[\mathcal{P}\right]]:=(h(\Delta),\cdots,h(W\Delta))=\bm{h}\left[\mathcal{P}\right]$. The sampling operator induces a functional $f_c:=f_d\circ \mathcal{S}$. 

We denote the space of discrete binning values as $\mathcal{H}_d$ and the space of continuous binning values as $\mathcal{H}_c$, and represent their relationship using the following commutative diagram:
\begin{equation}
    \label{cd:continuousconstruction}
    \begin{tikzcd}
                                                        & \mathcal{H}_c \arrow[d] \arrow[rd, "f_c"] \arrow[d, "\mathcal{S}"] &            \\
\mathcal{P} \arrow[r, "\boldsymbol{h}"] \arrow[ru, "h"] & \mathcal{H}_d \arrow[r] \arrow[r, "f_d"]                           & \mathbb{R}
\end{tikzcd}
\end{equation}

To compute $\frac{\partial f_c}{\partial x_i'}$, functional backpropagation computes the cotangent vector $v_{x_i'}$ from $v_f$ by:
\begin{equation}
    v_h(x)=v_f\frac{\delta f_c}{\delta h(x)},\quad v_{x_i'}=\int_\mathbb{R}v_{h}(x)\frac{\partial h(x)}{\partial x_i'}dx,
\end{equation}
The cotangent function $v_h(\cdot)$ can also be computed during the normal backpropagation as $f_c=f_d\circ\mathcal{S}$. Using the chain rule, we have:
\begin{equation}\label{eq:sampling}
    v_h(x)=\sum_{j=1}^Wv_{h_j}\frac{\delta h_j}{\delta h(x)}=\sum_{j=1}^Wv_{h_j}\delta(x-j\Delta),
\end{equation}
where $v_{h_j}$ is the classical cotangent vector of $f_d$ during backward pass and $\delta(\cdot)$ is the Dirac delta. \red{In signal processing terms}, \Cref{eq:sampling} states that a normal backward pass of the binning function $\bm{h}$ surrogates the cotangent function with a Dirac comb modulated by the cotangent vector. Fixing $W\Delta$, this surrogate becomes exact as $\Delta\to0$, and at this time, we have:
\begin{equation}\label{eq:diraccombeq}
    \lim_{\Delta\to0}\frac{1}{\Delta}v_{h_j}=v_h(j\Delta).
\end{equation}
This result connects the cotangent function and cotangent vector by the sampling rule, so we can use normal signal reconstruction methods to reconstruct the cotangent function and derive an exact formula for $\frac{\partial f_c}{\partial x_i'}$. \red{Since cotangent functions naturally encode smooth motion flow}, one simple solution is to replace $\delta(\cdot)$ in \Cref{eq:sampling} with another kernel $\frac{1}{\Delta}l(\frac{\cdot}{\Delta})$, then we have:
\begin{equation}
    \tilde{v}_{x_i'}=\sum_{j=1}^W\int_\mathbb{R}v_{h_j}\frac{w_i'}{\Delta^2}l'(\frac{x-j\Delta}{\Delta})k(\frac{x_i'-x}{\Delta})dx\overset{(a)}{=}\sum_{j=1}^Ww_i'\frac{\partial}{\partial x_i'}\kappa(\frac{x_i'-j\Delta}{\Delta})v_{h_j},
\end{equation}
where $\tilde{v}_{x_i'}$ is the synthesized cotangent vector, \red{(a) relies on integration-by-parts}, $l'(x)=\frac{dl(x)}{dx}$ and
\begin{equation}\label{eq:kerconv}
    \kappa(x) = \int_\mathbb{R}l(y)k(x-y)dy=(l*k)(x).
\end{equation}
In other words, the synthesized weak derivative of $h_j$ is
\begin{equation}\label{eq:synderivative}
    \tilde{\frac{\partial h_j}{\partial x_i'}}=w_i'\frac{\partial}{\partial x_i'}\kappa(\frac{x_i'-j\Delta}{\Delta}).
\end{equation}
Comparing \Cref{eq:synderivative} with the formal derivative $\frac{\partial h_j}{\partial x_i'}=w_i'\frac{\partial}{\partial x_i'}k(\frac{x_i'-j\Delta}{\Delta})$, we see that it's equal to a surrogate gradient of kernel $\kappa(\cdot)$. 
It's possible to demonstrate the unbiasedness of our method using the concept of weak derivatives, compared to other heuristic surrogates. \red{In fact, it has finite-order approximation precision to long-range finite differences of arbitrary targets. Detailed proofs are provided in the appendix.}
\red{For the multidimensional case in \Cref{eq:ndhistdef}, the result also holds when to replace $k_d$ with $\kappa_d$ for each dimension. Implementation is provided in the following section.}

\subsection{\red{Algorithm Implementation}}

\red{To simplify the implementation of FBP despite the theoretical density, we observe that our method only modifies the backward pass gradient accumulation. To demonstrate how different dimensions interact, a 2D example \Cref{alg:fbp} illustrates that the discontinuous kernel derivative is substituted with the synthesized weak derivative $\kappa'(\cdot)$, while the forward pass remains unchanged. }

\begin{algorithm}[h]
\caption{2D Functional Binning: Primal, Forward Mode (JVP), and Backward Mode (VJP)}
\label{alg:fbp}
\SetAlgoLined
\DontPrintSemicolon
\KwIn{Locations $\mathbf{x}'_i = (x'_i, y'_i)$, weights $w'_i$, Grid params ($W \times H$, spacing $\Delta$).}
\KwIn{Kernels: Binning $k$, Reconstruction $l$. \textbf{Optional:} Tangents $\dot{\mathbf{x}}'_i = (\dot{x}'_i, \dot{y}'_i)$, Adjoints $\bar{H}$.}
\KwOut{Frame $H \in \mathbb{R}^{W \times H}$. \textbf{Optional:} Tangent Frame $\dot{H}$, Gradients $\nabla_{\mathbf{x}'} \mathcal{L}$.}
\BlankLine
$\kappa(u) \leftarrow (l * k)(u)$; $\quad \kappa'(u) \leftarrow \frac{d}{du}\kappa(u)$; \tcp*{Synthesize derivative kernel}
\BlankLine
$H \leftarrow \text{Zeros}(W, H)$; \For(\tcp*[h]{1. Primal Pass: Standard Binning}){$i \leftarrow 1$ \KwTo $N$}{
\ForEach{bin index $(u, v)$ in support of $k$ around $(x'_i, y'_i)$}{
$d_x \leftarrow (x'_i - u \Delta) / \Delta$; $\quad d_y \leftarrow (y'_i - v \Delta) / \Delta$ \tcp*{Norm. distances}
$H[u, v] \leftarrow H[u, v] + w'_i \cdot k(d_x) \cdot k(d_y)$;
}
}
\BlankLine
\If(\tcp*[h]{2. Forward Mode: Tangent Propagation}){Tangents $\dot{\mathbf{x}}'$ are provided}{
$\dot{H} \leftarrow \text{Zeros}(W, H)$;
\For{$i \leftarrow 1$ \KwTo $N$}{
\ForEach{bin index $(u, v)$ in support of $\kappa$ around $(x'_i, y'_i)$}{
$d_x \leftarrow (x'_i - u \Delta) / \Delta$; $\quad d_y \leftarrow (y'_i - v \Delta) / \Delta$;
$val_x \leftarrow \frac{1}{\Delta} \kappa'(d_x) \cdot \kappa(d_y) \cdot \dot{x}'_i$;
$val_y \leftarrow \kappa(d_x) \cdot \frac{1}{\Delta} \kappa'(d_y) \cdot \dot{y}'_i$;
$\dot{H}[u, v] \leftarrow \dot{H}[u, v] + w'_i \cdot (val_x + val_y)$;
}
}
}
\BlankLine
\If(\tcp*[h]{3. Backward Mode: Adjoint Propagation}){Adjoints $\bar{H}$ are provided}{
$\nabla_{\mathbf{x}'} \mathcal{L} \leftarrow \text{Zeros}(N, 2)$;
\For{$i \leftarrow 1$ \KwTo $N$}{
$g_x \leftarrow 0$; $\quad g_y \leftarrow 0$;
\ForEach{bin index $(u, v)$ in support of $\kappa$ around $(x'_i, y'_i)$}{
$d_x \leftarrow (x'_i - u \Delta) / \Delta$; $\quad d_y \leftarrow (y'_i - v \Delta) / \Delta$;
$g_x \leftarrow g_x + \bar{H}[u, v] \cdot w'_i \cdot \frac{1}{\Delta} \kappa'(d_x) \cdot \kappa(d_y)$;
$g_y \leftarrow g_y + \bar{H}[u, v] \cdot w'_i \cdot \kappa(d_x) \cdot \frac{1}{\Delta} \kappa'(d_y)$;
}
$(\nabla_{\mathbf{x}'} \mathcal{L})_i \leftarrow (g_x, g_y)$;
}
}
\Return $H, \dot{H}, \nabla_{\mathbf{x}'} \mathcal{L}$
\end{algorithm}

\subsection{Bias Analysis}

\red{Before applying the method to complex tasks, we explicitly validate its ability to approximate long-range finite differences. Following the methodology of \citep{gallego2019focus}, we proceed with the following steps for an event packet $\{e_k=(x_k,y_k,t_k,p_k)\}_{k=1}^{N_e}$.}

\begin{enumerate}
    \item Events are warped to the mean time using rotational models:
    \begin{equation}\label{eq:rotwarp}
        T_{rot}:e_k\mapsto(t_{ref}-t)\boldsymbol\omega\times\bm{x}_k+\bm{x}_k,\qquad\bm{x}_k=(x_k,y_k,1)\mapsto \bm{x}_k'.
    \end{equation}
    \item An Image of Warped Events (IWE) of resolution $200\times 150$ is constructed using non-differentiable (\emph{rect}), differentiable (\emph{linear}), and biased (\emph{gauss}) binning kernels:
    \begin{equation}\hspace{-0.5cm}\label{eq:kernel}
    k_{rect}(x)=\mathbbm{1}_{|x|<\frac{1}{2}},\quad
    k_{linear}(x)=(1-|x|)\mathbbm{1}_{|x|<1},\quad
    k_{gauss}(x)=\tfrac{1}{\sqrt{2\pi}}e^{-x^2/2}\mathbbm{1}_{|x|<\tfrac32}.
    \end{equation}
    \item Synthesized gradients are computed with $l(x)=\max(1-|x|,0)$ with respect to Variance (Var) \citep{gallego2018unifying} and Log-Likelihood (LL) \citep{gu2021spatio} scores:
    \begin{equation}\label{eq:score}
        \text{Var}=\tfrac{1}{N_p}\sum_{i,j}(h_{i,j}-\mu_{H})^2,\quad
        \text{LL}=\sum_{i,j}\log(\text{NB}(h_{i,j}|r,p)),
    \end{equation}
    where $\mu_H$ is the frame mean, and $(r,p)=(0.3,0.8)$. Theoretical results show that this rectangular kernel has second-order accuracy in approximating the gradient of a score.
\end{enumerate}

We analyze bias using 20,000 events from the \emph{dynamic\_rotation} sequence of the Event Camera Dataset (ECD) \citep{mueggler2017event}. A uniform grid of candidate angular velocities is sampled within $[-5,5]^3$, yielding 1,331 score evaluations and 3,993 analytic gradients. Finite-difference bias is estimated by subtracting numerical gradients \red{computed via central difference with a step size of $1.0$}. The results are shown in \Cref{fig:biasanalysis}.


\begin{figure}[h]
    \centering
    \begin{subfigure}[t]{.5\textwidth}
        \centering
        \includegraphics[width=\linewidth]{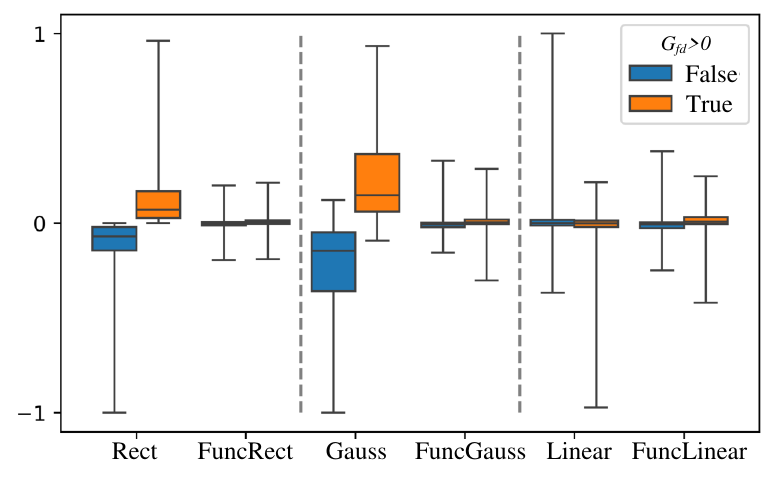}
        \caption{Normalized Var score bias.}
    \end{subfigure}%
    ~ 
    \begin{subfigure}[t]{.5\textwidth}
        \centering
        \includegraphics[width=\linewidth]{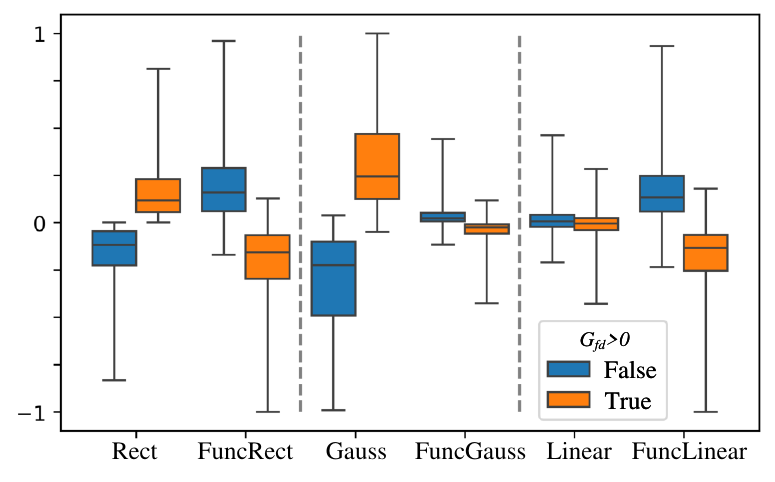}
        \caption{Normalized LL score bias.}
    \end{subfigure}
    \caption{Bias analysis results. The analytical gradients are subtracted by numerical gradients to obtain the finite-difference bias, \red{\textbf{colored by the sign of finite-difference gradients $G_{fd}$}}.\label{fig:biasanalysis}}
\end{figure}

As shown in \Cref{fig:biasanalysis}(a), the variance score exhibits strong bias with the \emph{rect} and \emph{gauss} kernels due to discontinuities, while the \emph{linear} kernel remains unbiased. Our gradient estimation substantially reduces bias across all kernels, even improving the \emph{linear} kernel by better approximating long-range finite differences.

For the log-likelihood score (\Cref{fig:biasanalysis}(b)), results are mixed: \emph{rect} kernel bias is reduced at the cost of higher variance, while the \emph{linear} kernel bias increases. \red{Consistent with our theoretical analysis}, this reflects the limited second-order accuracy of the triangular reconstruction $l(x)$ in approximating the nonlinear LL score. This suggests the exact gradient estimation method should be application-dependent, which our framework supports via adaptable cotangent reconstruction constraints.


\section{\red{Experiments}\label{sec:analysis}}
This section further validates the proposed method on event-based motion estimation tasks, specifically angular and linear velocity estimation, where the computed gradients guide the optimization. Besides the rotational motion model defined in \Cref{eq:rotwarp}, we introduce a linear motion model:
\begin{equation}\label{eq:transwarp}
    T_{trans}:e_k\mapsto(t_{ref}-t)\bm{v}+\bm{x}_k,~ \bm{x}_k=(x_k,y_k,1)\mapsto \bm{x}_k'.
\end{equation}
We evaluate all model–kernel–score combinations on eight sequences from the Event Camera Dataset (ECD), namely \textit{boxes, dynamic, poster}, and \textit{shapes} (both rotation and translation). The sequences are partitioned into non-overlapping packets of $N_e=20,000$ events to estimate motion parameters. Accuracy is quantified using Root Mean Square (RMS) error in \textdegree/s for rotation and m/s for translation. Images of Warped Events (IWEs) are constructed from projected coordinates with a resolution of $200\times 150$ and a binning step of $\Delta=0.01$. Experiments were conducted on a laptop equipped with an NVIDIA RTX 4090 GPU using the JAX framework \citep{jax2018github}. We employed distinct optimizers tailored to the kernel properties: L-BFGS-B \citep{liu1989limited} for the \emph{rect} and \emph{linear} kernels, and trust-ncg \citep{steihaug1983conjugate} for the \emph{gauss} kernel.

\subsection{Optimization Results}\label{sec:expopt}

\begin{figure}
    \centering
    \includegraphics[width=.98\linewidth]{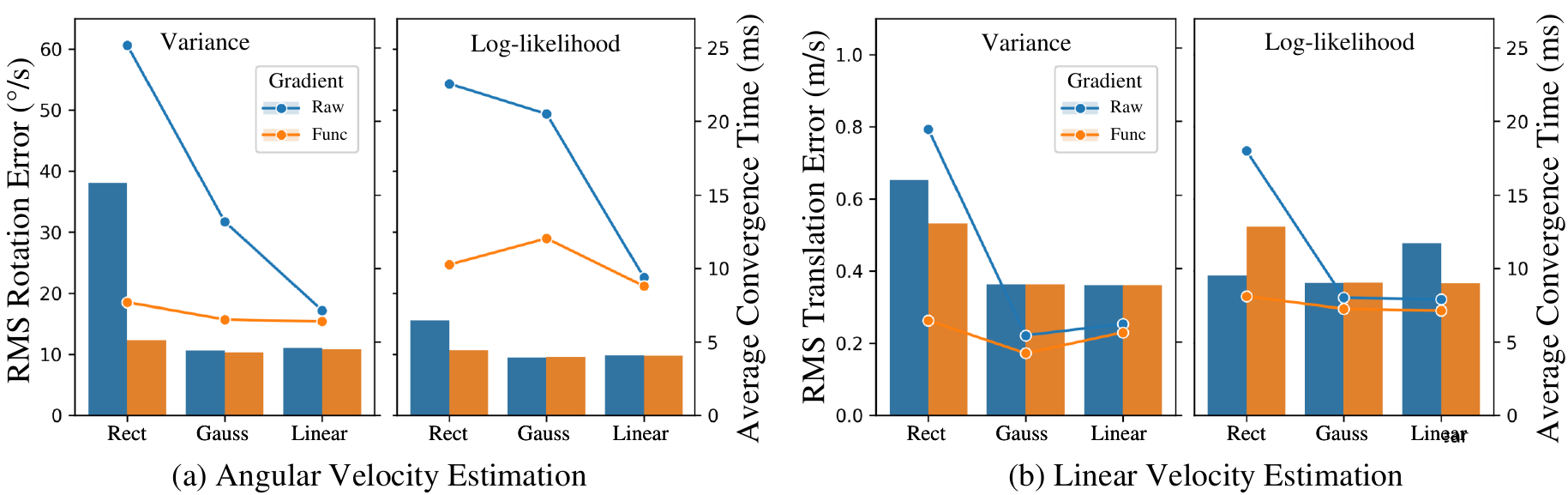}
    \caption{Optimization results for motion estimation, with combined bar charts showing RMS estimation accuracy and line charts showing the mean convergence time for every $N_e=20000$ events.}
    \label{fig:motionopt}
\end{figure}

The mean performance across all sequences is presented in \Cref{fig:motionopt}. It is evident that our method improves both accuracy and convergence speed for angular velocity estimation, achieving a 10.3\% reduction in RMS error and a $1.66\times$ acceleration in convergence. In the case of linear velocity, singularities of this task \citep{guo2024cmax} lead to a marginal 3.9\% increase in RMS error, yet the method maintains a $1.48\times$ faster convergence rate. Overall, the proposed method yields a 3.2\% improvement in RMS error alongside a $1.57\times$ speedup. Since our approach modifies only the gradients without altering the solution space, these performance gains indicate that our synthesized gradients provide more effective update directions. Although the gradient computation incurs higher theoretical complexity (\Cref{eq:kerconv}), the empirical results demonstrate that this cost is outweighed by the benefits in optimization efficiency. For complete results, please refer to the appendix.


\subsection{\red{Ablation Studies and Baseline Comparisons}}
\red{This subsection investigates the efficacy of Functional Backpropagation (FBP) using the linear reconstruction kernel. To isolate the effects of the reconstruction method, we focus on the rotational case utilizing the $k_{rect}$ binning kernel and the Var and LL score function.}



\textbf{\red{Sensitivity to Reconstruction Kernels.}} \red{For the reconstruction kernel $l(x)$, we additionally evaluate the Cubic kernel (assuming a $C^2$ smooth cotangent function) and the Lanczos kernel (assuming a band-limited cotangent function), defined as:}
\begin{equation}
\begin{gathered}
    l_{Cubic}(x)=\begin{cases}1.5|x|^3-2.5|x|^2+1 & |x|<1\\2.5|x|^2-0.5|x|^3-4|x|+2& 1<|x|\leq 2\\0&\text{otherwise}\end{cases},\\ l_{Lanczos}(x)=
        \frac{2\sin(\pi x)\sin(\frac{\pi}{2}x)}{\pi^2x^2}\mathbbm{1}_{|x|\leq 2}.
\end{gathered}
\end{equation}


\begin{table}[h]
\small
\setlength{\tabcolsep}{1.5pt}
\centering
\caption{\textbf{Ablation study on reconstruction kernels.} Different cotangent reconstruction kernels are compared in angular velocity estimation in Accuracy/Time (\textdegree/s and ms). Bold: best value.}
\label{tab:reconkernels}
\begin{tabular}{l|cc|cc|cc|cc}\toprule
\multirow{2}{*}{Kernel} & \multicolumn{2}{c|}{boxes\_rotation} & \multicolumn{2}{c|}{dynamic\_rotation} & \multicolumn{2}{c|}{poster\_rotation} & \multicolumn{2}{c}{shapes\_rotation} \\ \cmidrule{2-9} 
                        & Var             & LL                 & Var                & LL                & Var              & LL                 & Var               & LL               \\ \midrule
Bicubic                 & 12.47/7.37    & 10.17/9.17       & 6.70/8.56        & 5.88/10.78      & 14.39/7.61     & 12.48/9.54       & 16.47/11.31   & \textbf{13.82}/  18.07  \\
Lanzcos                 & 12.46/8.84    & \textbf{10.10}/10.12    & 6.72/10.36       & \textbf{5.86}/12.37      & \textbf{14.15}/9.01     & 12.45/10.66    & 17.93/14.83   & 14.25/21.12  \\
Linear                  & \textbf{12.44}/\textbf{6.50}    & 10.14/\textbf{8.07}       & \textbf{6.67}/\textbf{7.44}        & 5.88/\textbf{9.57}       & 14.26/\textbf{6.55}     & \textbf{12.42}/\textbf{8.17}       & \textbf{15.89}/\textbf{10.26}   & 14.26/\textbf{15.18} \\\bottomrule
\end{tabular}
\end{table}

\red{As shown in \Cref{tab:reconkernels}, the Linear kernel achieves the optimal trade-off between accuracy and efficiency. While higher-order kernels offer greater smoothness, they increase computational runtime with diminishing returns in estimation accuracy. Consequently, we utilize the Linear reconstruction kernel for the remainder of our experiments.}



\textbf{\red{Comparison vs. Heuristic Surrogate Gradients (SG).}} \red{We compare FBP against standard Surrogate Gradients commonly used in Spiking Neural Networks: the Straight-Through-Estimator (STE) \citep{yin2018understanding} and the Sigmoid Surrogate \citep{neftci2019surrogate}, defined as:}
\begin{equation}
    \kappa_{STE}'=-\text{sgn}(x)\mathbbm{1}_{|x|<1}, \kappa_{sigmoid}'=\sigma'(10(x+\frac{1}{2})-\sigma'(10(x-\frac{1}{2})), \sigma(x)=\frac{1}{1+\exp(-x)}.
\end{equation}

\begin{table}[h]
\small
\setlength{\tabcolsep}{1.5pt}
\centering
\caption{\textbf{Comparisons with heuristic surrogate gradients.} Different gradient surrogates are compared in angular velocity estimation in Accuracy/Time (\textdegree/s and ms). Bold: Best value.}
\label{tab:surrogategradients}
\begin{tabular}{l|cc|cc|cc|cc}\toprule
\multirow{2}{*}{Surrogate} & \multicolumn{2}{c|}{boxes\_rotation} & \multicolumn{2}{c|}{dynamic\_rotation} & \multicolumn{2}{c|}{poster\_rotation} & \multicolumn{2}{c}{shapes\_rotation} \\ \cmidrule{2-9} 
                           & Var               & LL               & Var                & LL                & Var               & LL                & Var               & LL               \\ \midrule
STE                        & 36.29/7.18      & 14.32/8.88     & 10.92/8.45       & 9.00/10.36      & 29.00/7.38      & 16.22/8.91      & 96.89/  13.07   & 55.83/  17.35  \\
Sigmoid                    & 13.94/7.55      & 10.17/8.49     & 7.12/8.74        & 5.88/9.84       & 15.60/7.70      & 12.51/8.68      & 18.76/11.41     & \textbf{14.23}/15.44    \\
FBP (Ours)                 & \textbf{12.44}/\textbf{6.50}      & \textbf{10.14}/\textbf{8.07}     & \textbf{6.67}/\textbf{7.44}        &\textbf{ 5.88}/\textbf{9.57}       & \textbf{14.26}/\textbf{6.55}      & \textbf{12.42}/\textbf{8.17}      & \textbf{15.89}/\textbf{10.26}     & 14.26/\textbf{15.18}   \\\bottomrule
\end{tabular}
\end{table}

\red{From \Cref{tab:surrogategradients}, FBP achieves a significantly lower error floor and reduced convergence time compared to heuristic SGs. This confirms that deriving the gradient via integration by parts captures the underlying binning geometry more effectively than imposing arbitrary smooth shapes.}

\subsection{\red{Computational Complexity}}

\red{To supplement the optimization analysis in \Cref{sec:expopt}, we verify the computational cost of the proposed method by comparing the execution time of a single call to the original binning function versus the gradient computation function. As Jacobian-vector products (JVPs) inherit the structure of the forward pass without requiring intermediate storage, we present the computation time for a single JVP call in float32 on both CPU and GPU in \Cref{tab:compcomplex}, where our method for different binning kernels is prefixed with Func. The data indicates that the computation time approximately doubles compared to the standard deduced JVP. However, when combined with the optimization results in \Cref{sec:expopt}, the overall efficiency and effectiveness of the method are confirmed.}

\begin{table*}[h]
    \centering
    \small
    \setlength{\tabcolsep}{5pt}
    \caption{Computation time in microseconds. Our modified gradient computation is shown in \textbf{bold.}}
\begin{tabular}{p{0.2\linewidth}|
                    >{\centering\arraybackslash}p{0.10\linewidth}
                    >{\centering\arraybackslash}p{0.10\linewidth}
                    >{\centering\arraybackslash}p{0.10\linewidth}|
                    >{\centering\arraybackslash}p{0.10\linewidth}
                    >{\centering\arraybackslash}p{0.10\linewidth}
                    >{\centering\arraybackslash}p{0.10\linewidth}}
    \toprule
    \multirow{1}{*}{Platforms} 
        & \multicolumn{3}{c|}{CPU (R9-7945HX)} 
        & \multicolumn{3}{c}{GPU (4090-laptop)} \\
    \cmidrule(lr){1-7}
     $N_e$   & 20000 & 50000 & 100000 
        & 20000 & 50000 & 100000 \\
    \midrule
    Rect                    & 37.1  & 69.9  & 110   & 31.1  & 29.6  & 29.7  \\
    \textbf{FuncRect JVP}   & \textbf{635}  & \textbf{893}  & \textbf{1440} & \textbf{42.3} & \textbf{62.4} & \textbf{78.8} \\
    \midrule
    Gauss                   & 536   & 820   & 1380  & 44.9  & 61.7  & 75.6  \\
    Gauss JVP               & 638   & 946   & 1600  & 40.6  & 62.3  & 77.7  \\
    \textbf{FuncGauss JVP}  & \textbf{1090} & \textbf{2130} & \textbf{5890} 
                            & \textbf{83.5} & \textbf{114}  & \textbf{158} \\
    \textbf{Overhead} & \textbf{1.71$\times$} & \textbf{2.25$\times$} & \textbf{3.68$\times$} & \textbf{2.06$\times$} & \textbf{1.83$\times$}  & \textbf{2.03$\times$}                      \\
    \midrule
    Linear                  & 473   & 551   & 774   & 26.8  & 37.1  & 55.6  \\
    Linear JVP              & 440   & 543   & 772   & 28.1  & 37.3  & 51.5  \\
    \textbf{FuncLinear JVP} & \textbf{736} & \textbf{1190} & \textbf{2120} 
                            & \textbf{66.8} & \textbf{83.1} & \textbf{117} \\
    \textbf{Overhead} & \textbf{1.67$\times$} & \textbf{2.19$\times$} & \textbf{2.75$\times$} & \textbf{2.38$\times$} & \textbf{2.23$\times$} & \textbf{2.27$\times$}                       \\
    \bottomrule
    \end{tabular}
    \label{tab:compcomplex}
\end{table*}

\section{Applications\label{sec:applications}}
To test the applicability of the proposed gradient computation method on modern learning pipelines, we conduct further experiments on the latest event-based optical flow estimation network MotionPriorCMax \citep{hamann2024motion} and SLAM algorithm CMax-SLAM \citep{guo2024cmax}.

\subsection{MotionPriorCMax}
MotionPriorCMax (MPC) is \red{the SOTA} self-supervised optical flow estimation network that uses the linear binning kernel $k_2$ to construct a contrast loss for learning, which is trained against the latest DSEC \citep{gehrig2021dsec} dataset. We changed the binning kernel to $k_{rect}$ using the modified gradient computation method and retrained the model on eight RTX A6000 GPUs. Reported metrics are End Point Error (EPE), Angular Error (AE), and percentage of outliers with EPE$>$3px (3PE).

\textbf{Results.} We present the average DSEC benchmark performance in \Cref{fig:appof} where our method improves the EPE error by 9.4\%, confirming that the new gradient computation method helps to find the optimal parameters even for a complex neural network. A visualization presented in the right of \Cref{fig:appof}, where optical flow output by ERAFT \citep{gehrig2021raft} is regarded as ground truth since it is not available on test sequences. It's observed that the network learns more robust features with a sharper IWE, which is only trainable with our method. For a complete result on all 7 test sequences, please refer to the appendix.




\begin{figure}[t]
    \centering
    \begin{minipage}[t]{0.38\textwidth}
        \vspace{9pt}
        \begin{tabular}{l|ccc}
        \toprule
        Method         & EPE $\downarrow$ & AE $\downarrow$ & 3PE $\downarrow$ \\ \midrule
        MPC            & 2.81            & 8.96           & 14.5            \\
        Ours   & \textbf{2.54}   & \textbf{8.33}  & \textbf{13.3}   \\   \bottomrule
        \end{tabular}
    \end{minipage}%
    \hfill
    \begin{minipage}[t]{0.6\textwidth}
        \vspace{0pt} 
        \includegraphics[width=\linewidth]{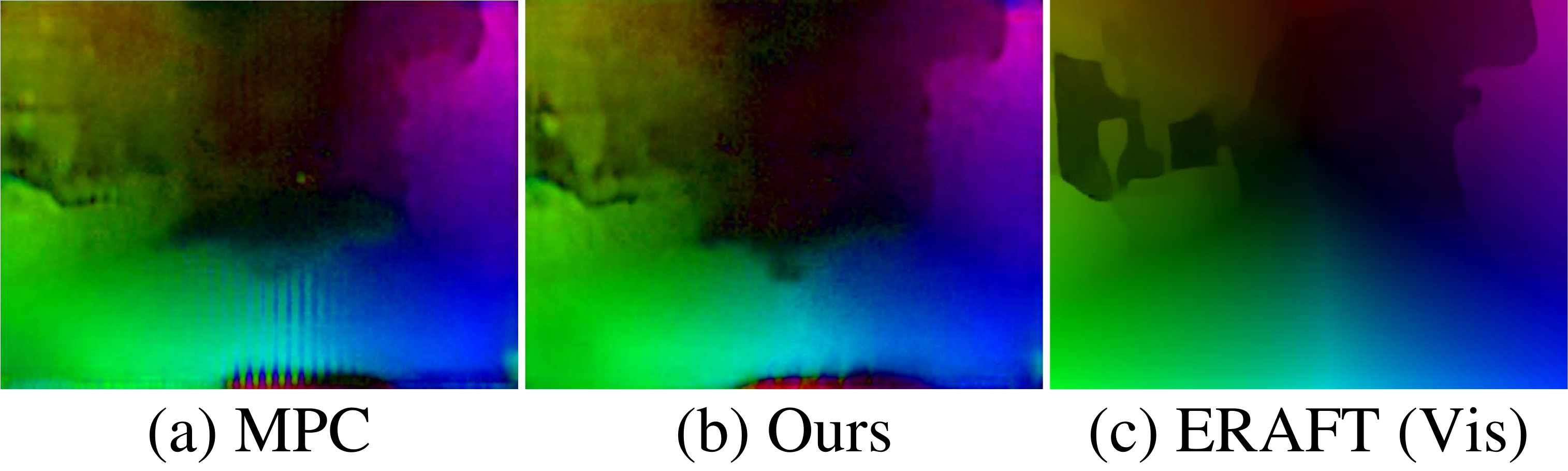}
    \end{minipage}

    \caption{\textbf{Left: }all sequence average performance on DSEC. \textbf{Right: }Predicted optical flow \red{(ERAFT only for qualitative visualization)}. Our method exhibits more robustness with fewer artifacts.}
    \label{fig:appof}
\end{figure}

\subsection{CMax-SLAM}
CMax-SLAM \citep{guo2024cmax} is \red{the SOTA} rotation-only SLAM algorithm that uses the Contrast Maximization framework \citep{gallego2018unifying} on both the frontend and backend, where a panoramic IWE is constructed as the map. We changed the binning kernel from $k_{linear}$ to $k_{rect}$ and report the RMS absolute trajectory error (Abs) in [\textdegree] and the RMS relative error (Rel) in [\textdegree/s] for the first 30s of data of the sequences in ECD \citep{mueggler2017event} for a fair comparison.

\textbf{Results.} The overall ECD benchmark results are shown in \Cref{fig:appcmaxslam}, where our method mainly improves the Abs RMS error by 5.1\% with comparable Rel RMS error. On the right side of \Cref{fig:appcmaxslam}, we show the reconstruction results of the two methods: the baseline fails on long sequences due to accumulated errors, while our method preserves correctness and yields more robust reconstructions, showing that sharp IWE helps reduce the mapping ambiguity.

\begin{figure*}[t]
    \centering
    \begin{minipage}[t]{0.58\textwidth}
        \vspace{7pt}
        \small
        \begin{tabular}{p{0.22\linewidth}|
                        >{\centering\arraybackslash}p{0.03\linewidth}
                        >{\centering\arraybackslash}p{0.048\linewidth}|
                        >{\centering\arraybackslash}p{0.03\linewidth}
                        >{\centering\arraybackslash}p{0.048\linewidth}|
                        >{\centering\arraybackslash}p{0.03\linewidth}
                        >{\centering\arraybackslash}p{0.048\linewidth}|
                        >{\centering\arraybackslash}p{0.03\linewidth}
                        >{\centering\arraybackslash}p{0.048\linewidth}}
        \toprule
        \multirow{2}{*}{Method} 
          & \multicolumn{2}{c|}{shapes} 
          & \multicolumn{2}{c|}{poster} 
          & \multicolumn{2}{c|}{boxes} 
          & \multicolumn{2}{c}{dynamic} \\
        \cmidrule{2-9}
          & Abs & Rel & Abs & Rel & Abs & Rel & Abs & Rel \\
        \midrule
        CMax-SLAM      
          & 4.95 & \textbf{6.71} 
          & 5.65 & 6.36 
          & 5.42 & 6.75 
          & 3.38 & \textbf{3.59} \\
        Ours 
          & \textbf{4.84} & 6.76 
          & \textbf{5.27} & \textbf{6.29} 
          & \textbf{5.23} & \textbf{6.60} 
          & \textbf{3.11} & 3.61 \\
        \bottomrule
        \end{tabular}
    \end{minipage}%
    \hfill
    \begin{minipage}[t]{0.4\textwidth}
        \vspace{0pt}
        \includegraphics[width=\linewidth]{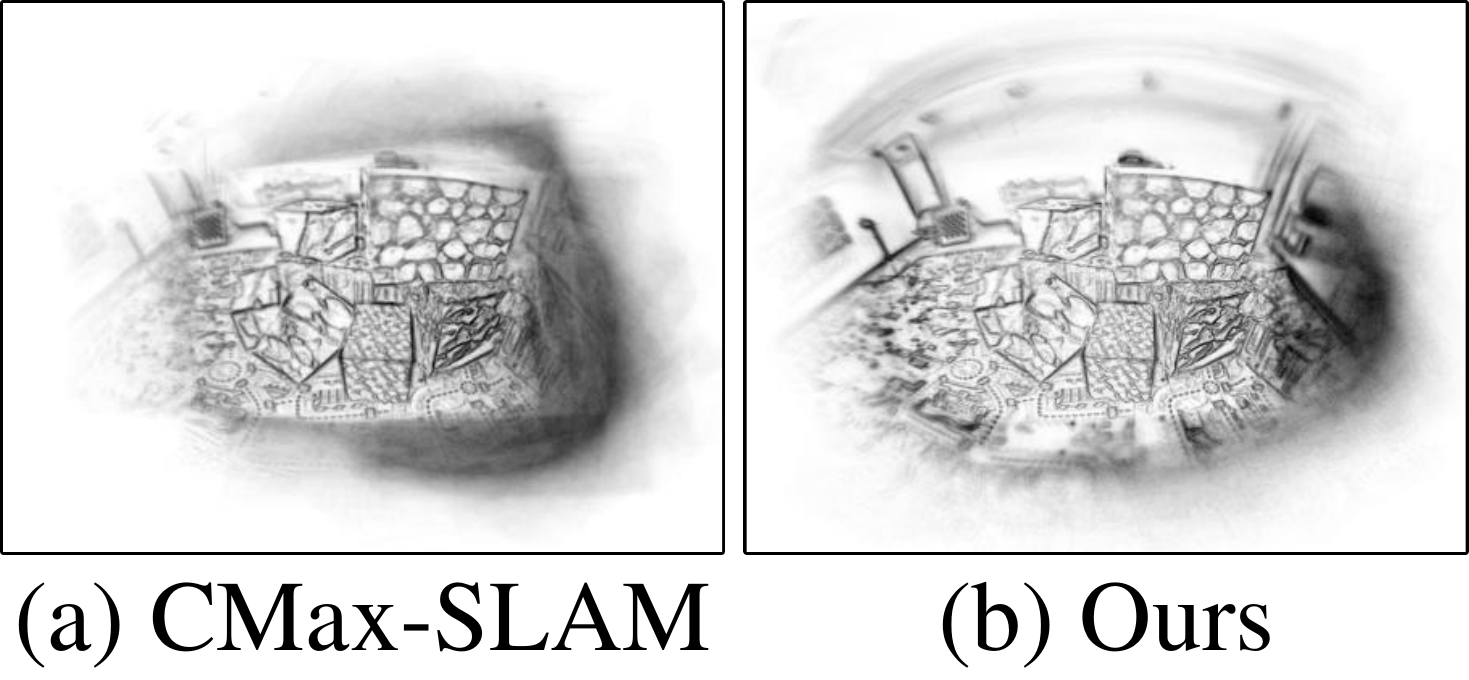}
    \end{minipage}

    \caption{\textbf{Left: }CMax-SLAM results on real-world sequences from ECD. 
             We report Absolute (Abs) and Relative (Rel) trajectory errors (best results in \textbf{bold}). 
             \textbf{Right: }Mapping result on the full sequence, 
             showing that CMax-SLAM fails on long sequences while our method preserves correctness.}
    \label{fig:appcmaxslam}
\end{figure*}

\section{Conclusion\label{sec:conclusion}}
In this work, we addressed the long-standing challenge of biased gradient estimation in event binning for event-based vision. By lifting binning functions to the functional space and leveraging weak derivatives through integration by parts, we derived an exact formula for synthesizing unbiased gradients. This framework, termed \textbf{Functional Backpropagation}, reconstructs cotangent functions during backpropagation and guarantees unbiasedness without altering the forward outputs. Our theoretical analysis established the correctness and unbiasedness of the synthesized gradients, while extensive experiments on motion estimation, optical flow, and SLAM demonstrated consistent improvements in both accuracy and convergence speed. Looking forward, our framework opens up opportunities to extend unbiased gradient estimation to a wider class of discontinuous operators in neuromorphic computing \red{and to generalizations to nonuniform binning grids and nonlinear reconstruction methods.}



\section*{Acknowledgements}
This work is supported by the National Natural Science Foundation of China (NSFC) under Grants 62225207, 62436008, 62306295, and 62576328. The AI-driven experiments, simulations and model training were performed on the robotic AI-Scientist platform of Chinese Academy of Sciences.

\bibliography{iclr2026_conference}

@inproceedings{gehrig2021raft,
  title={E-raft: Dense optical flow from event cameras},
  author={Gehrig, Mathias and Millh{\"a}usler, Mario and Gehrig, Daniel and Scaramuzza, Davide},
  booktitle={2021 International Conference on 3D Vision (3DV)},
  pages={197--206},
  year={2021},
  organization={IEEE}
}

@article{shiba2023event,
  title={Event-based background-oriented schlieren},
  author={Shiba, Shintaro and Hamann, Friedhelm and Aoki, Yoshimitsu and Gallego, Guillermo},
  journal={IEEE transactions on pattern analysis and machine intelligence},
  volume={46},
  number={4},
  pages={2011--2026},
  year={2023},
  publisher={IEEE}
}

@INPROCEEDINGS{zhu2018ev, 
    AUTHOR    = {Alex Zhu AND Liangzhe Yuan AND Kenneth Chaney AND Kostas Daniilidis}, 
    TITLE     = {EV-FlowNet: Self-Supervised Optical Flow Estimation for Event-based Cameras}, 
    BOOKTITLE = {Proceedings of Robotics: Science and Systems}, 
    YEAR      = {2018}, 
    ADDRESS   = {Pittsburgh, Pennsylvania}, 
    MONTH     = {June}, 
    DOI       = {10.15607/RSS.2018.XIV.062} 
}

@article{hines2025compact,
  title={A compact neuromorphic system for ultra--energy-efficient, on-device robot localization},
  author={Hines, Adam D and Milford, Michael and Fischer, Tobias},
  journal={Science Robotics},
  volume={10},
  number={103},
  pages={eads3968},
  year={2025},
  publisher={American Association for the Advancement of Science}
}

@article{vidal2018ultimate,
  title={Ultimate SLAM? Combining events, images, and IMU for robust visual SLAM in HDR and high-speed scenarios},
  author={Vidal, Antoni Rosinol and Rebecq, Henri and Horstschaefer, Timo and Scaramuzza, Davide},
  journal={IEEE Robotics and Automation Letters},
  volume={3},
  number={2},
  pages={994--1001},
  year={2018},
  publisher={IEEE}
}

@article{zhou2021event,
  title={Event-based stereo visual odometry},
  author={Zhou, Yi and Gallego, Guillermo and Shen, Shaojie},
  journal={IEEE Transactions on Robotics},
  volume={37},
  number={5},
  pages={1433--1450},
  year={2021},
  publisher={IEEE}
}

@inproceedings{hamann2024motion,
  title={Motion-prior contrast maximization for dense continuous-time motion estimation},
  author={Hamann, Friedhelm and Wang, Ziyun and Asmanis, Ioannis and Chaney, Kenneth and Gallego, Guillermo and Daniilidis, Kostas},
  booktitle={European Conference on Computer Vision},
  pages={18--37},
  year={2024},
  organization={Springer}
}

@article{rebecq2019high,
  title={High speed and high dynamic range video with an event camera},
  author={Rebecq, Henri and Ranftl, Ren{\'e} and Koltun, Vladlen and Scaramuzza, Davide},
  journal={IEEE transactions on pattern analysis and machine intelligence},
  volume={43},
  number={6},
  pages={1964--1980},
  year={2019},
  publisher={IEEE}
}

@inproceedings{tulyakov2021time,
  title={Time lens: Event-based video frame interpolation},
  author={Tulyakov, Stepan and Gehrig, Daniel and Georgoulis, Stamatios and Erbach, Julius and Gehrig, Mathias and Li, Yuanyou and Scaramuzza, Davide},
  booktitle={Proceedings of the IEEE/CVF conference on computer vision and pattern recognition},
  pages={16155--16164},
  year={2021}
}

@inproceedings{pan2019bringing,
  title={Bringing a blurry frame alive at high frame-rate with an event camera},
  author={Pan, Liyuan and Scheerlinck, Cedric and Yu, Xin and Hartley, Richard and Liu, Miaomiao and Dai, Yuchao},
  booktitle={Proceedings of the IEEE/CVF conference on computer vision and pattern recognition},
  pages={6820--6829},
  year={2019}
}

@inproceedings{xu2025motion,
  title={Motion-adaptive Transformer for Event-based Image Deblurring},
  author={Xu, Senyan and Sun, Zhijing and Zhong, Mingchen and Cao, Chengzhi and Liu, Yidi and Fu, Xueyang and Chen, Yan},
  booktitle={Proceedings of the AAAI Conference on Artificial Intelligence},
  volume={39},
  number={9},
  pages={8942--8950},
  year={2025}
}

@article{neftci2019surrogate,
  title={Surrogate gradient learning in spiking neural networks: Bringing the power of gradient-based optimization to spiking neural networks},
  author={Neftci, Emre O and Mostafa, Hesham and Zenke, Friedemann},
  journal={IEEE Signal Processing Magazine},
  volume={36},
  number={6},
  pages={51--63},
  year={2019},
  publisher={IEEE}
}

@article{eshraghian2023training,
  title={Training spiking neural networks using lessons from deep learning},
  author={Eshraghian, Jason K and Ward, Max and Neftci, Emre O and Wang, Xinxin and Lenz, Gregor and Dwivedi, Girish and Bennamoun, Mohammed and Jeong, Doo Seok and Lu, Wei D},
  journal={Proceedings of the IEEE},
  volume={111},
  number={9},
  pages={1016--1054},
  year={2023},
  publisher={IEEE}
}

@inproceedings{
yin2018understanding,
title={Understanding Straight-Through Estimator in Training Activation Quantized Neural Nets},
author={Penghang Yin and Jiancheng Lyu and Shuai Zhang and Stanley J. Osher and Yingyong Qi and Jack Xin},
booktitle={International Conference on Learning Representations},
year={2019},
url={https://openreview.net/forum?id=Skh4jRcKQ},
}

@article{lagorce2016hots,
  title={Hots: a hierarchy of event-based time-surfaces for pattern recognition},
  author={Lagorce, Xavier and Orchard, Garrick and Galluppi, Francesco and Shi, Bertram E and Benosman, Ryad B},
  journal={IEEE transactions on pattern analysis and machine intelligence},
  volume={39},
  number={7},
  pages={1346--1359},
  year={2016},
  publisher={IEEE}
}

@inproceedings{maqueda2018event,
  title={Event-based vision meets deep learning on steering prediction for self-driving cars},
  author={Maqueda, Ana I and Loquercio, Antonio and Gallego, Guillermo and Garc{\'\i}a, Narciso and Scaramuzza, Davide},
  booktitle={Proceedings of the IEEE conference on computer vision and pattern recognition},
  pages={5419--5427},
  year={2018}
}

@inproceedings{gallego2018unifying,
  title={A unifying contrast maximization framework for event cameras, with applications to motion, depth, and optical flow estimation},
  author={Gallego, Guillermo and Rebecq, Henri and Scaramuzza, Davide},
  booktitle={Proceedings of the IEEE conference on computer vision and pattern recognition},
  pages={3867--3876},
  year={2018}
}

@inproceedings{gu2021spatio,
  title={The spatio-temporal poisson point process: A simple model for the alignment of event camera data},
  author={Gu, Cheng and Learned-Miller, Erik and Sheldon, Daniel and Gallego, Guillermo and Bideau, Pia},
  booktitle={Proceedings of the IEEE/CVF International Conference on Computer Vision},
  pages={13495--13504},
  year={2021}
}

@article{huh2018gradient,
  title={Gradient descent for spiking neural networks},
  author={Huh, Dongsung and Sejnowski, Terrence J},
  journal={Advances in neural information processing systems},
  volume={31},
  year={2018}
}

@article{frigyik2008introduction,
  title={An introduction to functional derivatives},
  author={Frigyik, B{\'e}la A and Srivastava, Santosh and Gupta, Maya R},
  journal={Dept. Electr. Eng., Univ. Washington, Seattle, WA, Tech. Rep},
  volume={1},
  year={2008}
}

@incollection{kuttler2017weak,
  title={Weak Derivatives},
  author={Kuttler, Kenneth},
  booktitle={Modern Analysis (1997)},
  pages={355--370},
  year={2017},
  publisher={CRC Press}
}

@software{jax2018github,
  author = {James Bradbury and Roy Frostig and Peter Hawkins and Matthew James Johnson and Chris Leary and Dougal Maclaurin and George Necula and Adam Paszke and Jake Vander{P}las and Skye Wanderman-{M}ilne and Qiao Zhang},
  title = {{JAX}: composable transformations of {P}ython+{N}um{P}y programs},
  url = {http://github.com/jax-ml/jax},
  version = {0.3.13},
  year = {2018},
}

@article{mueggler2017event,
  title={The event-camera dataset and simulator: Event-based data for pose estimation, visual odometry, and SLAM},
  author={Mueggler, Elias and Rebecq, Henri and Gallego, Guillermo and Delbruck, Tobi and Scaramuzza, Davide},
  journal={The International journal of robotics research},
  volume={36},
  number={2},
  pages={142--149},
  year={2017},
  publisher={SAGE Publications Sage UK: London, England}
}

@article{gehrig2021dsec,
  title={Dsec: A stereo event camera dataset for driving scenarios},
  author={Gehrig, Mathias and Aarents, Willem and Gehrig, Daniel and Scaramuzza, Davide},
  journal={IEEE Robotics and Automation Letters},
  volume={6},
  number={3},
  pages={4947--4954},
  year={2021},
  publisher={IEEE}
}

@article{liu1989limited,
  title={On the limited memory BFGS method for large scale optimization},
  author={Liu, Dong C and Nocedal, Jorge},
  journal={Mathematical programming},
  volume={45},
  number={1},
  pages={503--528},
  year={1989},
  publisher={Springer}
}

@article{steihaug1983conjugate,
  title={The conjugate gradient method and trust regions in large scale optimization},
  author={Steihaug, Trond},
  journal={SIAM Journal on Numerical Analysis},
  volume={20},
  number={3},
  pages={626--637},
  year={1983},
  publisher={SIAM}
}

@article{guo2024cmax,
  title={CMax-SLAM: Event-based rotational-motion bundle adjustment and SLAM system using contrast maximization},
  author={Guo, Shuang and Gallego, Guillermo},
  journal={IEEE Transactions on Robotics},
  volume={40},
  pages={2442--2461},
  year={2024},
  publisher={IEEE}
}

@article{gallego2020event,
  title={Event-based vision: A survey},
  author={Gallego, Guillermo and Delbr{\"u}ck, Tobi and Orchard, Garrick and Bartolozzi, Chiara and Taba, Brian and Censi, Andrea and Leutenegger, Stefan and Davison, Andrew J and Conradt, J{\"o}rg and Daniilidis, Kostas and others},
  journal={IEEE transactions on pattern analysis and machine intelligence},
  volume={44},
  number={1},
  pages={154--180},
  year={2020},
  publisher={IEEE}
}

@article{shi2024stochastic,
  title={Stochastic taylor derivative estimator: Efficient amortization for arbitrary differential operators},
  author={Shi, Zekun and Hu, Zheyuan and Lin, Min and Kawaguchi, Kenji},
  journal={Advances in Neural Information Processing Systems},
  volume={37},
  pages={122316--122353},
  year={2024}
}

@inproceedings{gallego2019focus,
  title={Focus is all you need: Loss functions for event-based vision},
  author={Gallego, Guillermo and Gehrig, Mathias and Scaramuzza, Davide},
  booktitle={Proceedings of the IEEE/CVF Conference on Computer Vision and Pattern Recognition},
  pages={12280--12289},
  year={2019}
}

@inproceedings{paszke2019pytorch,
  title={PyTorch: An Imperative Style, High-Performance Deep Learning Library},
  author={Paszke, Adam and Gross, Sam and Chintala, Soumith and Chanan, Gregory and Yang, Edward and DeVito, Zachary and Lin, Zeming and Desmaison, Alban and Antiga, Luca and Lerer, Adam},
  booktitle={Advances in Neural Information Processing Systems},
  volume={32},
  year={2019}
}

@inproceedings{luo2024efficient,
  title={Efficient Meshflow and Optical Flow Estimation from Event Cameras},
  author={Luo, Xinglong and Luo, Ao and Wang, Zhengning and Lin, Chunyu and Zeng, Bing and Liu, Shuaicheng},
  booktitle={Proceedings of the IEEE/CVF Conference on Computer Vision and Pattern Recognition},
  pages={19198--19207},
  year={2024}
}

@Article{Shiba24pami,
  author        = {Shintaro Shiba and Yannick Klose and Yoshimitsu Aoki and Guillermo Gallego},
  title         = {Secrets of Event-based Optical Flow, Depth, and Ego-Motion by Contrast Maximization},
  journal       = {IEEE Trans. Pattern Anal. Mach. Intell. (T-PAMI)},
  year          = 2024,
  pages         = {1--18},
  doi           = {10.1109/TPAMI.2024.3396116}
}

@article{luo2025learning,
  title={Learning Efficient Meshflow and Optical Flow from Event Cameras},
  author={Luo, Xinglong and Luo, Ao and Luo, Kunming and Wang, Zhengning and Tan, Ping and Zeng, Bing and Liu, Shuaicheng},
  journal={IEEE Transactions on Pattern Analysis and Machine Intelligence},
  year={2025},
  publisher={IEEE}
}

@article{ghosh2025event,
  title={Event-based stereo depth estimation: A survey},
  author={Ghosh, Suman and Gallego, Guillermo},
  journal={IEEE Transactions on Pattern Analysis and Machine Intelligence},
  year={2025},
  publisher={IEEE}
}

@inproceedings{teng2022nest,
  title={NEST: Neural event stack for event-based image enhancement},
  author={Teng, Minggui and Zhou, Chu and Lou, Hanyue and Shi, Boxin},
  booktitle={European Conference on Computer Vision},
  pages={660--676},
  year={2022},
  organization={Springer}
}

@inproceedings{
liao2025efdgs,
title={{EF}-3{DGS}: Event-Aided Free-Trajectory 3D Gaussian Splatting},
author={Bohao Liao and Wei Zhai and Zengyu Wan and Zhixin Cheng and Wenfei Yang and Yang Cao and Tianzhu Zhang and Zheng-Jun Zha},
booktitle={The Thirty-ninth Annual Conference on Neural Information Processing Systems},
year={2025},
url={https://openreview.net/forum?id=shFhW4zqd6}
}

@InProceedings{Wan_2025_ICCV,
    author    = {Wan, Zengyu and Zhai, Wei and Cao, Yang and Zha, Zhengjun},
    title     = {EMoTive: Event-guided Trajectory Modeling for 3D Motion Estimation},
    booktitle = {Proceedings of the IEEE/CVF International Conference on Computer Vision (ICCV)},
    month     = {October},
    year      = {2025},
    pages     = {9342-9351}
}

@InProceedings{Han_2025_ICCV,
    author    = {Han, Han and Zhai, Wei and Cao, Yang and Li, Bin and Zha, Zheng-jun},
    title     = {MATE: Motion-Augmented Temporal Consistency for Event-based Point Tracking},
    booktitle = {Proceedings of the IEEE/CVF International Conference on Computer Vision (ICCV)},
    month     = {October},
    year      = {2025},
    pages     = {8340-8349}
}

@article{wan2024event,
  title={Event-based optical flow via transforming into motion-dependent view},
  author={Wan, Zengyu and Tan, Ganchao and Wang, Yang and Zhai, Wei and Cao, Yang and Zha, Zheng-Jun},
  journal={IEEE Transactions on Image Processing},
  volume={33},
  pages={5327--5339},
  year={2024},
  publisher={IEEE}
}

@article{wu2024event,
  title={Event-based asynchronous HDR imaging by temporal incident light modulation},
  author={Wu, Yuliang and Tan, Ganchao and Chen, Jinze and Zhai, Wei and Cao, Yang and Zha, Zheng-Jun},
  journal={Optics Express},
  volume={32},
  number={11},
  pages={18527--18538},
  year={2024},
  publisher={Optica Publishing Group}
}
\bibliographystyle{iclr2026_conference}


\appendix
\section{Appendix}
\subsection{Use of LLMs}
This paper uses ChatGPT only for text polishing, correcting mathematical symbols, and searching relevant work. The surrogate gradient (SG) method is one that was new knowledge to the authors and provides useful insights into the presentation of this paper.

\subsection{Functional Automatic Differentiation}
This section supplements the discussion in \Cref{sec:fbp} to cover general automatic differentiation, and to make the text self-contained with the necessary mathematical results proved.

To simplify the expression of functional automatic differentiation, we introduce the following conventions:
\begin{itemize}
    \item \textbf{Vertical bar notation.} 
    For an operator $f$, we write
    \[
        f(u|y) := [f(u)](y),
    \]
    meaning that $f$ takes a function $u\in \mathcal{H}(X)$ as input, produces a function in $\mathcal{H}(Y)$, and then we evaluate this function at $y\in Y$. 
    This notation is non-standard, but convenient for distinguishing between function mapping and function evaluation. 
    Given no risk of ambiguity, we may write $f(y):=f(u|y)$ for brevity.
    
    \item \textbf{Variational symbol $\bm{\delta}$.} 
    We use $\delta$ in analogy to Leibniz’s differential $d$, but acting on functions instead of scalar values. 
    The quotient
    \[
        \frac{\delta f(y)}{\delta g(x)}
    \]
    denotes the \emph{operator derivative} of $f(y)$ with respect to $g(x)$. 
\end{itemize}


\subsubsection{The Chain Rule of Functionals}
Fundamental to automatic differentiation is the decomposition of differentials provided by the chain rule of partial derivatives of composite functions. However, since functionals are functions that consume an entire function to give a value, the input dimension is infinite, so the derivatives cannot be expressed by finite-dimensional Jacobian matrices. To extend the chain rule to functionals, remember that derivatives are unique local linear approximations to the function. This leads to the definition of the Fréchet derivative:

\begin{definition}[Fréchet]\label{def:derivative}
    Let $V$ and $W$ be normed vector spaces, and $U\subseteq V$ be an open subset of $V$. A function $f:U\to W$ is called Fréchet differentiable at $x\in U$ if there exists a bounded linear operator $A: V\to W$ such that
    \begin{equation}
        \lim_{\|h\|_V\to 0}\frac{\|f(x+h)-f(x)-Ah\|_W}{\|h\|_V}=0.
    \end{equation}
    If there exists such an operator $A$, it is unique, so we write $Df[x]=A$ and call it the Fréchet derivative of $f$ at $x$. A function $f$ is Fréchet differentiable if $\forall x\in U$, $Df[x]$ exists.
\end{definition}

For the Fréchet derivative, the well-known Fermat's theorem also applies to identify critical points of a given functional by solving an equation involving the derivative. This result makes it possible to treat the Fréchet derivative like an ordinary derivative when solving an optimization problem.
\begin{theorem}[Fermat]\label{theorem:functionalchainrule}
    Let $V$ be a normed vector space and $f:V\to\mathbb{R}$ is a functional that is Fréchet differentiable on $V$. Then if $x_0$ is a point where $f$ has a local extremum, $Df[x_0]=\mathbf{0}$ (the zero functional).
\end{theorem}
\begin{proof}
    Without loss of generality, let $x_0$ be the local maxima. Then there exists $\delta>0$ such that when $\|h\|_V<\delta$,
    \begin{equation}\label{eq:prooffermat}
        f(x_0+h)\leq f(x_0).
    \end{equation}
    According to \Cref{def:derivative}, we have
    \begin{equation}
        f(x_0+h)-f(x_0)-Df[x_0](h)=o(\|h\|_V).
    \end{equation}
    For all $v\in V$, we can find a $t$ that satisfy $\|tv\|_V<\delta$, substitute $h$ for $tv$:
    \begin{equation}
        f(x_0+tv)-f(x_0)=t\cdot Df[x_0](v)+o(|t|)\leq0.
    \end{equation}
    Taking $t\to0^+$, gives $Df[x_0](v)\leq 0$, likewise with $t\to0^-$, we have $Df[x_0](v)\geq0$. Thus $Df[x_0](v)=0,\forall v \in V$ and hence $Df[x_0]=0$.
\end{proof}

The chain rule has an analogous form as ordinary derivative, but expressed as composition of linear operators.
\begin{theorem}[The Chain Rule]\label{theorem:functionalchainrulerep}
    Let $U$, $V$ and $W$ be normed vector spaces and $U\overset{f}{\to}V\overset{g}{\to}W$. If $f$ is Fréchet differentiable at $u\in U$ and $g$ is Fréchet differentiable at $f(u)\in V$, then the composite function $h=g\circ f$ is Fréchet differentiable at $u$ with the chain rule
    \begin{equation}
        Dh[u]=Dg[f(u)]\circ Df[u].
    \end{equation}
\end{theorem}
\begin{proof}
    From \Cref{def:derivative}, we have for a small perturbation $\delta u$ such that $h$ is differentiable at $u+\delta u$,
    \begin{align}
        h(u+\delta u) &= g(f(u+\delta u)),\\
        &=g(f(u)+Df[u](\delta u) + r_f(\delta u)),\\
        &=h(u)+Dg[f(u)](Df[u](\delta u))+Dg[f(u)](r_f(\delta u))+r_g(Df[u](\delta u) + r_f(\delta u)),
    \end{align}
    such that $\lim_{\varepsilon\to0}\frac{\|r_f(\varepsilon)\|_V}{\|\varepsilon\|_U}=0,$ and $\lim_{\eta\to0}\frac{\|r_g(\eta)\|_W}{\|\eta\|_V}=0$. Since $Df[u]$ and $Dg[f(u)]$ are bounded, there exist constants $C_1$ and $C_2$ such that $\|Df[u](\varepsilon)\|_V\leq C_1\|\varepsilon\|_U$ and $\|Dg[f(u)](\eta)\|_W\leq C_2\|\eta\|_V$. So for the remainder, we have
    \begin{align}
        \frac{\|Dg[f(u)](r_f(\delta u))\|_W}{\|\delta u\|_U}\leq C_2\frac{\|r_f(\delta u)\|_V}{\|\delta u\|_U}\to 0
    \end{align}
    and
    \begin{align}
        \frac{\|r_g(Df[u](\delta u) + r_f(\delta u))\|_W}{\|\delta u\|_U}&=\frac{\|r_g(Df[u](\delta u) + r_f(\delta u))\|_W}{\|Df[u](\delta u) + r_f(\delta u)\|_V}\frac{\|Df[u](\delta u) + r_f(\delta u)\|_V}{\|\delta u\|_U},\\
        &\leq \frac{\|r_g(Df[u](\delta u) + r_f(\delta u))\|_W}{\|Df[u](\delta u) + r_f(\delta u)\|_V}(C_1+\frac{\|r_f(\delta u)\|_V}{\|\delta u\|_U})\to 0.
    \end{align}
    So $h(u+\delta u)=h(u)+(Dg[f(u)]\circ Df[u]) (\delta u)+o(\|\delta u\|_U)$.
\end{proof}

In automatic differentiation, closed-form expressions are required rather than abstract operators. To derive this, we add the additional constraint that all vector spaces are Hilbert spaces, so that we can apply the \textbf{Riesz Representation Theorem}. To simplify the expression, we introduce a new notation analogous to Leibniz's notation for operators.


\begin{theorem}[Riesz]
    Denote $H(X)=L^2(X)\cap C(X)$ as the space of continuous and square-integrable functions on $X$. Let $H(X)\overset{f}{\to}H(Y)\overset{g}{\to}H(Z)$, $f$ is Fréchet differentiable in $H(X)$ and $g$ is Fréchet differentiable in $H(Y)$. Denote point evaluation functionals on $H(Y)$ and $H(Z)$ as $L_y:v\mapsto v(y),\forall v\in H(Y)$ and $L_z:w\mapsto w(z),\forall w \in H(Z)$ respectively and the Riesz representation of $D(L_y\circ f)[u]$ as $\frac{\delta f(u|y)}{\delta I(u|x)}$, where $I$ represents the identity mapping, then the Riesz representation of composite functional $D(L_z\circ g\circ f)[u]$ is:
    \begin{equation}
        \frac{\delta g(f(u)|z)}{\delta I(u|x)}=\int_Y \frac{\delta g(f(u)|z)}{\delta f(u|y)}\frac{\delta f(u|y)}{\delta I(u|x)}dy.
    \end{equation}
\end{theorem}
\begin{proof}
    From the definition, $L_*$ is a linear operator so $DL_*[\cdot]=L_*$. $D(L_z\circ g\circ f)[u]=L_z\circ Dg[f(u)]\circ Df[u]$ from the chain rule, then $\forall v\in H(X)$, we have
    \begin{align}
        L_z\circ Dg[f(u)]\circ Df[u](v)&=\int_Y \frac{\delta g(f(u)|z)}{\delta f(u|y)}(Df[u](v))(y)dy,\\
        &=\int_Y \frac{\delta g(f(u)|z)}{\delta f(u|y)}(L_y\circ Df[u])(v)dy,\\
        &=\int_Y\frac{\delta g(f(u)|z)}{\delta f(u|y)}dy\int_X \frac{\delta f(u|y)}{\delta I(u|x)}v(x)dx,\\
        &=\int_{X\times Y}\frac{\delta g(f(u)|z)}{\delta f(u|y)}\frac{\delta f(u|y)}{\delta I(u|x)}v(x)dxdy.
    \end{align}
    Since $v\in H(X)$ is arbitrary, the theorem is proved.
\end{proof}


%
\subsubsection{Interpretation of Automatic Differentiation of Functionals}
In ordinary AD, the task is to compute Jacobian-Vector products (forward mode) or Vector-Jacobian products (reverse mode) so that entries of the full Jacobian matrix can be computed by selecting the test vectors to be the standard basis $\bm{e}^{(i)}$ in the tangent or cotangent spaces. In functional AD, the analogue is the integration of the product of the functional derivative representation and the tangent (cotangent) function.
\begin{definition}[Functional Forward Mode AD]
    Given a Fréchet differentiable function $f:H(X)\to H(Y)$ and tangent function at $u\in H(x)$ as $v\in T_uH(X)$, the functional forward mode AD program computes the tangent function at $f(u)\in H(Y)$ as $w\in T_{f(u)}H(Y)$.
    \begin{equation}
        w(y)=\int_X\frac{\delta f(u|y)}{\delta I(u|x)}v(x)dx.
    \end{equation}
\end{definition}

\begin{definition}[Functional Reverse Mode AD]
    Given a Fréchet differentiable function $f:H(X)\to H(Y)$ and cotangent function at $f(u)\in H(Y)$ as $w^*\in T^*_{f(u)}H(Y)$, the functional reverse mode AD program computes the cotangent function at $u\in H(X)$ as $v^*\in T^*_uH(X)$.
    \begin{equation}
        v^*(x)=\int_Y\frac{\delta f(u|y)}{\delta I(u|x)}w^*(y)dy.
    \end{equation}
\end{definition}

Now, with the formal definition of functional AD, it's possible to derive a recursive procedure to compute the Riesz representation of composite functionals:
\begin{corollary}
    Let $f_i:H(X_i)\to H(X_{i+1})$ be Fréchet differentiable in $H(X_i)$ and $F_{n}=f_{n-1}\circ f_{n-2}\circ...\circ f_1$. The functional forward mode AD states that the tangent function of $F_{n}$ at $F_{n}(u)$ can be recursively computed as:
    \begin{equation}
        w(x_{m+1})=\int_{X_{m}}\frac{\delta f_{m}(F_{m}(u)|x_{m+1})}{\delta f_{m-1}(F_{m-1}(u)|x_m)}w(x_{m})dx_m
    \end{equation}
    and the functional reverse mode AD states that the cotangent function of $F_{n}$ at $u$ can be recursively computed as:
    \begin{equation}
        v^*(x_{m-1})=\int_{X_m}\frac{\delta f_{m-1}(F_{m-1}(u)|x_m)}{\delta f_{m-2}(F_{m-2}(u)|x_{m-1})}v^*(x_m)dx_m
    \end{equation}
    where we define $f_{-1}=f_0=I$ as the identity map so $F_0=F_1=I$.
\end{corollary}

\subsubsection{The Connection Between Functions and Sampled Dirac Combs}
In \Cref{eq:diraccombeq}, we make an equation between the cotangent function and the sampled cotangent vector. This is understood in the distributional sense. In fact, we have:
\begin{theorem}
For any smooth function $g(x)\in\mathcal{H}(\mathbb{R})$, if \Cref{eq:diraccombeq} hold, we have:
\begin{equation}
    \lim_{\Delta\to 0}\int g(x)v_h(x)dx=\sum_{j=1}^W\int g(x)v_{h_j}\delta(x-j\Delta)dx.
\end{equation}    
\end{theorem}
\begin{proof}
Select simple functions $g(x)=\max(1-|\frac{x-j\Delta}{\Delta}|,0)$, the equation is proved by observing that the right-hand side constitutes the Riemann sum. As any smooth function can be approximated by summing such simple functions, the result is proved.
\end{proof}

\subsection{The Unbiasedness Constraint}\label{sec:unbiasednessconstraint}

The problem of ensuring unbiasedness is how to quantify the bias when the true gradients don't exist. To achieve that, we resort to line integrals so that the bias can be measured in the distributional sense. Let $\phi$ be an arbitrary smooth function that accepts a binned event frame $\bm{h}\in H_d$ to return a scalar $\phi(\bm{h})$, then by the fundamental theorem of calculus for line integrals, we have for every piecewise-smooth curve $\bm{h}:[a,b]\to H_d$:
\begin{equation}\label{eq:lineintegral}
    \phi(\bm{h}(b))-\phi(\bm{h}(a))=\int_a^b\nabla\phi(\bm{h}(t))\cdot \bm{h}'(t)dt,
\end{equation}
where $\bm{h}'(t)$ should be understood in the weak sense. Assuming we want to find the optimal histogram parameters by inferring the extreme values of some smooth objective, then the derivative $\bm{h}'(t)$ is the only function that makes the above equation true. In this view, if the above equation is nearly true with $\tilde{\bm{h}}'(t)$, then we call it an unbiased gradient estimation.

For a simplified equation that can be used to check the unbiasedness, we pick each entry $h_j(t)$ inside the vector dot product and approximate $\nabla_j\phi(\bm{h}(t))$ with another arbitrary smooth function $\varphi_j(t)$, then the above argument translates to:
\begin{equation}\label{eq:unbiasedness}
    \int \varphi_j(t)h_j'(t)dt=-\int\varphi_j'(t)h_j(t)dt,
\end{equation}
when $h_j\phi_n$ vanishes at the boundaries. Since there's no $\tilde{\bm{h}}'(t)$ other than $\bm{h}'(t)$ to ensure the equality for arbitrary $\varphi_i$ and curve $\bm{h}$, we only require the equality to hold for a family of simple functions on simple paths. In practice, we select the pow functions and paths on directly related parameters $x'_{i_d}$s. Since both sides of \Cref{eq:unbiasedness} are linear functionals of $\varphi_i$, equality only needs to hold for the elementary power functions. So we have the following definition.
\begin{definition}\label{def:unbiasedness}
    If for all polynomial functions $\varphi_n(t)$ of degree $n$, equality
    \begin{equation}
        \int_T \varphi_n(t)\tilde{h}'(t)dt=-\int_T\varphi_n'(t)h(t)dt
    \end{equation}
    holds for all $n\leq m$ but not for $n>m$, we say $\tilde{h}'(t)$ estimates the weak derivative $h'(t)$ with a degree of precision of $m$ on path $T$. If $m\geq 1$, we say the estimation is unbiased.
\end{definition}

The definition above quantifies how well the estimation of a weak derivative is. The definition of biasedness conforms to the fact that the derivative is the best local linear approximation. For a multidimensional gradient field approximated by the cumulative product of 1D functions, we have the same result if each function satisfies \Cref{def:unbiasedness}.

It's now easy to derive the unbiasedness constraint of the proposed method. In fact, we have the following result.

\begin{corollary}
    \label{corollary:unbiasedbinning}
    The synthetic weak derivative $\tilde{\frac{\partial h_j}{\partial x_i'}}$ in \Cref{eq:synderivative} is an unbiased estimation of the weak derivative $\frac{\partial h_j}{\partial x_i'}$ on path $T$ if the reconstruction kernel $l(\cdot)$ satisfies
    \begin{equation}
        \label{eq:unbiasedness}
        \int_Tk(x)dx=\int_T\int_{\mathbb{R}}l(y)k(x-y)dydx.
    \end{equation}
\end{corollary}

\Cref{corollary:unbiasedbinning} provides a minimum requirement for the gradient estimation to be unbiased. However, since the binning function is evaluated on a regular grid, the neighborhood information of $h$ can provide useful information when selecting $l(\cdot)$ instead of a random guess. In many cases, $\phi(\cdot)$ attributes to $h_i$ a similar contribution as $h_j$ if $i$ is in a neighborhood of $j$, then we have $v_{h_i}\approx v_{h_j}$. In other words, we can assume the underlying cotangent function space to be smooth, such that local information provides enough characterization. One simple assumption is that the underlying space should have a minimum slope, then we drive the linear spline kernel $l(x)=\max(1-|x|,0)$, which has a minimum support such that \Cref{eq:unbiasedness} holds for a short path $T$. In fact, It can be proved that the resulting $(l * k)'(x)$ provides a second-order accurate estimate of the weak derivative $k'(x)$ along any path containing $\text{supp}(k) \cap [-1,1]$, provided that $k(\cdot)$ is symmetric with respect to some vertical axis. We conduct the following experiments on this linear reconstruction kernel.


\subsection{\red{PyTorch/JAX Implementation of FBP}}\label{sec:sourcecode}
This part provides the detailed implementation in PyTorch and JAX for implementing the proposed FBP.

PyTorch \citep{paszke2019pytorch} is based on reverse-mode automatic differentiation (VJP). A function with custom AD rules should be implemented as a subclass of torch.autograd.Function as follows.

\begin{lstlisting}[language=Python, caption=PyTorch implementation of FBP as custom VJP rules.]
import torch

class FunctionalBinning(torch.autograd.Function):
    @staticmethod
    def forward(ctx, xarray, yarray, weights, grid_size):
        ctx.save_for_backward(xarray, yarray, weights)
        # Standard binning with k (Forward pass is unchanged!)
        return k_kernel_binning(xarray, yarray, weights, grid_size)

    @staticmethod
    def backward(ctx, grad_frame):
        xarray, yarray, weights = ctx.saved_tensors
        # Use synthesized transposed derivative with kappa (Eq. 14)
        xarray_dot = kappa_kernel_dx_transposed(\
                grad_frame, xarray, yarray, grid_size)
        yarray_dot = kappa_kernel_dy_transposed(\
                grad_frame, xarray, yarray, grid_size)
        return xarray_dot * weights, yarray_dot * weights, None, None
\end{lstlisting}
The only extra work to do is to derive the transposed derivative kernel of $\kappa(\cdot)$. JAX \citep{jax2018github} is based on forward-mode automatic differentiation (JVP). A function with custom AD rules should be implemented by decorating it with custom\_jvp and implementing the custom JVP rules as follows.

\begin{lstlisting}[language=Python, caption=JAX implementation of FBP as custom JVP rules.]
import jax

@jax.custom_jvp
def FunctionalBinning(xarray, yarray, weights, grid_size):
    # Standard binning with k (Forward pass is unchanged!)
    return k_kernel_binning(xarray, yarray, weights, grid_size)

@FunctionalBinning.defjvp
def FunctionalBinning_JVP(primals, tangents):
    xarray, yarray, weights, grid_size = primals
    xarray_dot, yarray_dot, *_ = tangents

    # 1. Run standard forward pass
    frame = FunctionalBinning(xarray, yarray, weights, grid_size)

    # 2. Compute synthesized derivative using kappa (Eq. 14)
    # instead of evaluating Dirac deltas.
    frame_dot = kappa_kernel_dx_binning(\
            xarray, yarray, xarray_dot * weights, grid_size) +\
                kappa_kernel_dy_binning(\
            xarray, yarray, yarray_dot * weights, grid_size)

    return frame, frame_dot
\end{lstlisting}


\subsection{Experimental Details}
We first provide the exact formula of $\kappa_{rect}$, $\kappa_{linear}$, and $\kappa_{gauss}$, which are used to surrogate the gradient rule of the respective kernels $k_{rect}$, $k_{linear}$, and $k_{guass}$:
\begin{equation}
    \kappa_{rect}(x)=\begin{cases}
        \frac34-x^2, &|x|<\frac12\\
        \frac18(3-2|x|)^2, &\frac12\leq|x|<\frac32\\
        0. &|x|\geq\frac32
    \end{cases}
\end{equation}
\begin{equation}
    \kappa_{linear}(x)=\begin{cases}
        \frac16(4+3(-2+|x|)|x|^2), &|x|<1\\
        \frac16(2-|x|)^3, &1\leq|x|<2\\
        0. &|x|\geq2
    \end{cases}
\end{equation}
\begin{equation}
    \kappa_{gauss}(x)=\begin{cases}
        \frac{1}{2} (x-1) \text{erf}\left(\frac{x-1}{\sqrt{2}}\right)-x \text{erf}\left(\frac{x}{\sqrt{2}}\right)\\+\frac{1}{2} (x+1) \text{erf}\left(\frac{x+1}{\sqrt{2}}\right)+\frac{e^{-\frac{1}{2} (x+1)^2} \left(e^{2 x}-2 e^{x+\frac{1}{2}}+1\right)}{\sqrt{2 \pi }}, & |x|<\frac12\\
        \frac{1}{2} \left((|x|-1) \text{erf}\left(\frac{|x|-1}{\sqrt{2}}\right)-2x \text{erf}\left(\frac{x}{\sqrt{2}}\right)+\text{erf}\left(\frac{3}{2 \sqrt{2}}\right) |x|\right.\\\left.+\text{erf}\left(\frac{3}{2 \sqrt{2}}\right)-2 \sqrt{\frac{2}{\pi }} e^{-\frac{x^2}{2}}+\sqrt{\frac{2}{\pi }} e^{-\frac{1}{2} (|x|-1)^2}+\frac{\sqrt{\frac{2}{\pi }}}{e^{9/8}}\right), & \frac12\leq|x|<\frac32\\
        \frac{1}{2} (|x|-1) \text{erf}\left(\frac{x-1}{\sqrt{2}}\right)-\frac{1}{2} \text{erf}\left(\frac{3}{2 \sqrt{2}}\right) (|x|-1)\\+\frac{e^{-\frac{1}{2} (|x|-1)^2}}{\sqrt{2 \pi }}-\frac{1}{e^{9/8} \sqrt{2 \pi }}, & \frac32\leq|x|<\frac52\\
        0. & |x|\geq\frac52
    \end{cases}
\end{equation}

Although the expressions are definite, they are scaled to match the exact binning function, so they are adaptive to the specific scene.

\subsection{Additional Experimental Results}\label{sec:allresults}

Due to space limitations, detailed results for the analysis and applications are presented here. \Cref{tab:angvel} and \Cref{tab:linvel} provide the optimization results for all the sequences, where our methods are in bold with optimizer L-BFGS-B suffixed by 1 and trust-ncg suffixed by 2. These tables demonstrate another benefit of our method, which is to provide second-order information to facilitate optimization.

\begin{table*}[ht]
    \centering
    \small
    \caption{Angular velocity estimation results. Accuracy in \textdegree/s. The optimization time is in milliseconds. n/a means the optimization process has early stops due to abnormal function values.}
    \begin{tabular}{>{\centering\arraybackslash}p{0.05\linewidth}| 
                    >{\centering\arraybackslash}p{0.14\linewidth}|
                    >{\centering\arraybackslash}p{0.055\linewidth}
                    >{\centering\arraybackslash}p{0.07\linewidth}
                    >{\centering\arraybackslash}p{0.055\linewidth}
                    >{\centering\arraybackslash}p{0.07\linewidth}|
                    >{\centering\arraybackslash}p{0.055\linewidth}
                    >{\centering\arraybackslash}p{0.07\linewidth}
                    >{\centering\arraybackslash}p{0.055\linewidth}
                    >{\centering\arraybackslash}p{0.07\linewidth}}
    \toprule
    \multirow{3}{*}[0.1em]{Score} & \multirow{3}{*}[0.1em]{Methods} 
        & \multicolumn{2}{c}{boxes} & \multicolumn{2}{c}{dynamic} 
        & \multicolumn{2}{c}{poster} & \multicolumn{2}{c}{shapes} \\
    \cmidrule(lr){3-10}
        & & acc & time & acc & time & acc & time & acc & time \\
    \midrule
    \multirow{10}{*}{Var} 
        & Rect                  & 44.61 & 22.42 & 16.03 & 23.38 & 47.09 & 22.60 & 44.47 & 32.34 \\
        & \textbf{FuncRect 1}   & 12.44 & 6.50  & 6.67  & 7.44  & 14.26 & 6.55  & 15.89 & 10.26 \\
        & \textbf{FuncRect 2}   & 10.91 & 3.00  & 6.11  & 4.24  & 12.83 & 3.30  & 23.94 & 9.96  \\
        & Linear                & 10.21 & 6.38  & 5.63  & 6.60  & 12.58 & 6.36  & 15.82 & 9.18  \\
        & \textbf{FuncLinear 1} & 9.82  & 5.26  & 5.53  & 5.82  & 12.16 & 5.39  & 15.75 & 9.06  \\
        & \textbf{FuncLinear 2} & 9.66  & 2.61  & 5.50  & 4.77  & 12.01 & 3.19  & 14.93 & 13.07 \\
        & Gauss 1               & 10.93 & 7.14  & 6.59  & 8.38  & 12.51 & 7.41  & 14.29 & 11.13 \\
        & Gauss 2               & 9.61  & 5.01  & 6.22  & 8.36  & 10.80 & 5.49  & 15.79 & 33.87 \\
        & \textbf{FuncGauss 1}  & 10.27 & 6.91  & 6.22  & 7.98  & 11.91 & 7.22  & 13.81 & 11.45 \\
        & \textbf{FuncGauss 2}  & 9.95  & 2.46  & 5.98  & 5.59  & 11.59 & 3.26  & 13.85 & 14.78 \\
    \midrule
    \multirow{10}{*}{LL}  
        & Rect                  & 17.47 & 21.02 & 8.58  & 21.50 & 17.45 & 20.00 & 18.62 & 27.73 \\
        & \textbf{FuncRect 1}   & 10.14 & 8.07  & 5.88  & 9.57  & 12.42 & 8.17  & 14.26 & 15.18 \\
        & \textbf{FuncRect 2}   & 10.95 & 4.31  & 8.48  & 8.73  & 13.80 & 5.67  & n/a   & n/a   \\
        & Linear                & 9.04  & 7.29  & 5.33  & 8.93  & 11.53 & 7.98  & 13.36 & 13.30 \\
        & \textbf{FuncLinear 1} & 8.77  & 6.74  & 5.21  & 7.97  & 11.30 & 6.71  & 13.74 & 13.72 \\
        & \textbf{FuncLinear 2} & 8.70  & 3.29  & 5.21  & 7.35  & 11.22 & 4.05  & 14.62 & 34.21 \\
        & Gauss 1               & 8.99  & 8.68  & 5.64  & 10.28 & 10.97 & 8.80  & 13.46 & 15.89 \\
        & Gauss 2               & 7.93  & 5.95  & 5.36  & 12.10 & 9.97  & 6.85  & 14.54 & 57.10 \\
        & \textbf{FuncGauss 1}  & 8.89  & 8.64  & 5.49  & 10.17 & 10.91 & 8.59  & 13.37 & 16.75 \\
        & \textbf{FuncGauss 2}  & 8.64  & 3.80  & 5.35  & 9.51  & 10.74 & 5.30  & 13.41 & 29.52 \\
    \bottomrule
    \end{tabular}
    \label{tab:angvel}
\end{table*}

\begin{table*}[ht]
    \centering
    \small
    \caption{Linear velocity estimation results. Accuracy in m/s. The optimization time is in milliseconds. n/a means the optimization process has early stops due to abnormal function values.}
    \begin{tabular}{>{\centering\arraybackslash}p{0.05\linewidth}| 
                    >{\centering\arraybackslash}p{0.14\linewidth}|
                    >{\centering\arraybackslash}p{0.055\linewidth}
                    >{\centering\arraybackslash}p{0.07\linewidth}
                    >{\centering\arraybackslash}p{0.055\linewidth}
                    >{\centering\arraybackslash}p{0.07\linewidth}|
                    >{\centering\arraybackslash}p{0.055\linewidth}
                    >{\centering\arraybackslash}p{0.07\linewidth}
                    >{\centering\arraybackslash}p{0.055\linewidth}
                    >{\centering\arraybackslash}p{0.07\linewidth}}
    \toprule
    \multirow{3}{*}[0.1em]{Score} & \multirow{3}{*}[0.1em]{Methods} 
        & \multicolumn{2}{c}{boxes} & \multicolumn{2}{c}{dynamic} 
        & \multicolumn{2}{c}{poster} & \multicolumn{2}{c}{shapes} \\
    \cmidrule(lr){3-10}
        & & acc & time & acc & time & acc & time & acc & time \\
    \midrule
    \multirow{10}{*}{Var} 
        & Rect                  & 0.924 & n/a   & 0.447 & n/a   & 0.326 & 19.45 & 0.913 & n/a   \\
        & \textbf{FuncRect 1}   & 0.875 & n/a   & 0.209 & 7.93  & 0.304 & 6.45  & 0.739 & n/a   \\
        & \textbf{FuncRect 2}   & 0.679 & 3.40  & 0.208 & 5.18  & 0.304 & 3.55  & 0.257 & 10.73 \\
        & Linear                & 0.680 & 6.28  & 0.208 & 6.75  & 0.304 & 6.22  & 0.250 & 8.33  \\
        & \textbf{FuncLinear 1} & 0.683 & 5.63  & 0.208 & 6.78  & 0.303 & 5.64  & 0.249 & 8.16  \\
        & \textbf{FuncLinear 2} & 0.683 & 3.48  & 0.208 & 6.77  & 0.303 & 3.71  & 0.250 & 12.73 \\
        & Gauss 1               & 1.069 & n/a   & 0.209 & 8.14  & 0.803 & n/a   & 0.796 & n/a   \\
        & Gauss 2               & 0.689 & 5.26  & 0.209 & 8.67  & 0.305 & 5.44  & 0.249 & 24.64 \\
        & \textbf{FuncGauss 1}  & 0.688 & 6.84  & 0.209 & 8.81  & 0.304 & 7.02  & 0.248 & 11.45 \\
        & \textbf{FuncGauss 2}  & 0.689 & 3.85  & 0.209 & 8.38  & 0.305 & 4.22  & 0.249 & 15.35 \\
    \midrule
    \multirow{10}{*}{LL}  
        & Rect                  & 0.763 & 18.34 & 0.217 & 20.05 & 0.308 & 18.00 & 0.258 & 25.84 \\
        & \textbf{FuncRect 1}   & 0.681 & 7.97  & 0.432 & n/a   & 0.303 & 8.10  & 0.673 & n/a   \\
        & \textbf{FuncRect 2}   & 0.683 & 6.00  & 0.218 & 12.12 & 0.304 & 6.28  & 0.663 & n/a   \\
        & Linear                & 0.684 & 7.91  & 0.215 & 10.11 & 0.304 & 7.89  & 0.706 & n/a   \\
        & \textbf{FuncLinear 1} & 0.685 & 6.96  & 0.215 & 9.61  & 0.304 & 7.13  & 0.257 & 14.01 \\
        & \textbf{FuncLinear 2} & 0.685 & 6.00  & 0.216 & 11.47 & 0.304 & 5.30  & 0.256 & 28.93 \\
        & Gauss 1               & 1.108 & n/a   & 0.450 & n/a   & 0.304 & 8.76  & 0.861 & n/a   \\
        & Gauss 2               & 0.689 & 7.65  & 0.217 & 15.62 & 0.305 & 7.99  & 0.257 & 50.15 \\
        & \textbf{FuncGauss 1}  & 0.688 & 8.59  & 0.218 & 12.34 & 0.710 & n/a   & 0.795 & n/a   \\
        & \textbf{FuncGauss 2}  & 0.689 & 6.72  & 0.218 & 15.39 & 0.305 & 7.24  & 0.257 & 31.05 \\
    \bottomrule
    \end{tabular}
    \label{tab:linvel}
\end{table*}

For additional optical flow estimation results, refer to \Cref{tab:mpc_results} and \Cref{fig:appendixofall}. The results are consistent with the main text that our method improves the overall robustness.
\begin{table*}[ht]
    \centering
    \small
    \caption{Comparison of MPC and MPC (w/ FBP) across multiple sequences. 
    Metrics are End-Point Error (EPE), Angular Error (AE), and percentage of outliers (\%Out).}
    \begin{tabular}{p{0.14\linewidth}|
                    >{\centering\arraybackslash}p{0.035\linewidth}
                    >{\centering\arraybackslash}p{0.035\linewidth}
                    >{\centering\arraybackslash}p{0.045\linewidth}|
                    >{\centering\arraybackslash}p{0.035\linewidth}
                    >{\centering\arraybackslash}p{0.035\linewidth}
                    >{\centering\arraybackslash}p{0.045\linewidth}|
                    >{\centering\arraybackslash}p{0.035\linewidth}
                    >{\centering\arraybackslash}p{0.035\linewidth}
                    >{\centering\arraybackslash}p{0.045\linewidth}|
                    >{\centering\arraybackslash}p{0.035\linewidth}
                    >{\centering\arraybackslash}p{0.035\linewidth}
                    >{\centering\arraybackslash}p{0.045\linewidth}}
    \toprule
    \multirow{2}{*}{Method} 
      & \multicolumn{3}{c|}{All} 
      & \multicolumn{3}{c|}{interlaken\_00\_b} 
      & \multicolumn{3}{c|}{interlaken\_01\_a} 
      & \multicolumn{3}{c}{thun\_01\_a} \\
    \cmidrule(lr){2-13}
      & EPE & AE & \%Out 
      & EPE & AE & \%Out 
      & EPE & AE & \%Out 
      & EPE & AE & \%Out \\
    \midrule
    MPC & 2.81 & 8.96 & 14.48 & 3.33 & 5.11 & 19.41 & 2.42 & 6.17 & 16.15 & 1.48 & 6.70 & 8.33 \\
    MPC (w/ FBP) & 2.54 & 8.33 & 13.31 & 3.26 & 5.05 & 18.96 & 2.36 & 5.70 & 15.70 & 1.42 & 6.39 & 7.52 \\
    \midrule
    \multirow{2}{*}{Method} 
      & \multicolumn{3}{c|}{thun\_01\_b} 
      & \multicolumn{3}{c|}{zurich\_city\_12\_a} 
      & \multicolumn{3}{c|}{zurich\_city\_14\_c} 
      & \multicolumn{3}{c}{zurich\_city\_15\_a} \\
    \cmidrule(lr){2-13}
      & EPE & AE & \%Out 
      & EPE & AE & \%Out 
      & EPE & AE & \%Out 
      & EPE & AE & \%Out \\
    \midrule
    MPC & 1.47 & 5.88 & 8.18 & 5.98 & 21.56 & 21.02 & 1.97 & 9.74 & 15.25 & 1.54 & 6.71 & 9.25 \\
    MPC (w/ FBP) & 1.45 & 5.64 & 7.67 & 4.71 & 19.65 & 17.67 & 1.90 & 9.22 & 13.82 & 1.50 & 6.22 & 7.95 \\
    \bottomrule
    \end{tabular}
    \label{tab:mpc_results}
\end{table*}

\begin{figure*}
    \centering
    \includegraphics[width=.9\linewidth]{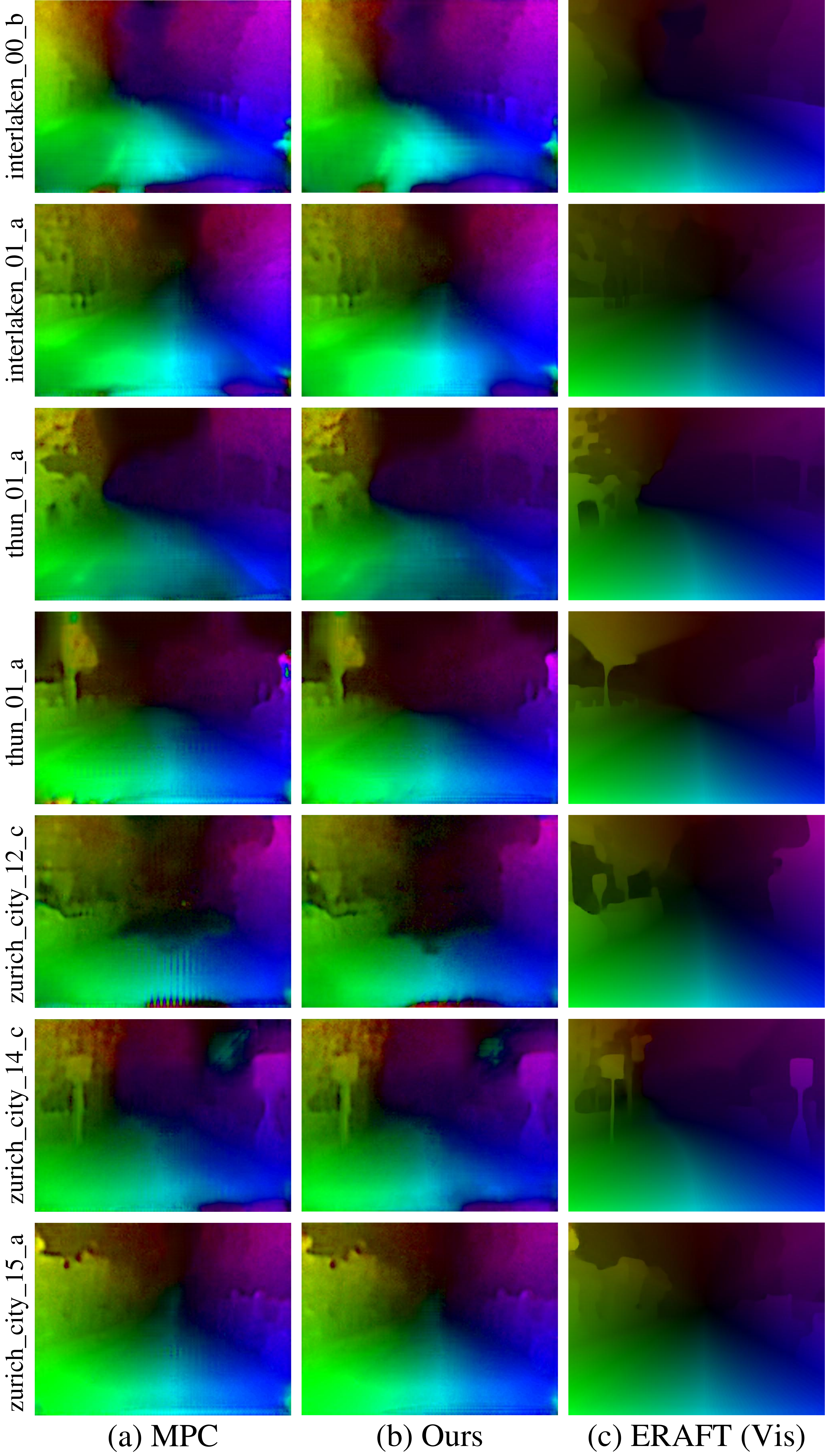}
    \caption{Optical flow estimation results on all 7 test sequences of DSEC \citep{gehrig2021dsec}.}
    \label{fig:appendixofall}
\end{figure*}

\end{document}